\documentclass[letterpaper]{article} 
\usepackage{aaai2027} 
\nocopyright
\usepackage[hyphens]{url}  
\usepackage{graphicx} 
\usepackage{natbib}  
\usepackage{caption} 
\usepackage{longtable}
\usepackage{algorithm}
\usepackage{algorithmic}

\usepackage{amsmath}
\usepackage{amssymb}
\usepackage{amsthm}
\usepackage{tikz}
\usetikzlibrary{arrows.meta,positioning,fit,backgrounds}

\usepackage{newfloat}
\usepackage{listings}
\DeclareCaptionStyle{ruled}{labelfont=normalfont,labelsep=colon,strut=off} 
\floatstyle{ruled}
\newfloat{listing}{tb}{lst}{}
\floatname{listing}{Listing}

\usepackage{booktabs}

\theoremstyle{plain}
\newtheorem{theorem}{Theorem}
\newtheorem{proposition}{Proposition}
\newtheorem{corollary}{Corollary}
\theoremstyle{definition}
\newtheorem{definition}{Definition}

\title{Swiss-Knife: A Framework for Reconfigurable Externalised\\Multi-Objective Alignment at Decode Time}
\author {
    Agnibh Karmakar\textsuperscript{\rm 1},
    Mayur Parvatikar\textsuperscript{\rm 2},
    Shreyash Dhoot\textsuperscript{\rm 2},
    Amit Dhanda\textsuperscript{\rm 4},
    Aman Chadha\textsuperscript{\rm 5},
    Kapil Wanaskar\textsuperscript{\rm 6},
    Vinija Jain\textsuperscript{\rm 3},
    Amitava Das\textsuperscript{\rm 7}
}

\affiliations {
    \textsuperscript{\rm 1}BITS Pilani Hyderabad Campus, India, 
    \textsuperscript{\rm 2}RV College of Engineering, India,  
    \textsuperscript{\rm 3}Google, USA,  
    \textsuperscript{\rm 4}Amazon, USA,  
    \textsuperscript{\rm 5}Apple, USA,  
    \textsuperscript{\rm 6}Canva, USA,  
    \textsuperscript{\rm 7}Pragya Lab, BITS Pilani Goa Campus, India
}
\begin{document}

\maketitle
\begin{abstract}
Decode-time alignment methods steer a frozen language model by scoring candidate continuations with an external reward and selecting the maximiser. We argue that this shared design is a single degenerate point in a much larger space. We introduce Swiss-Knife, a framework for externalised multi-objective alignment in which the alignment specification is a first-class runtime object: hot-swappable scoring \emph{blades}, a batch normaliser, a pairwise \emph{aggregation operator}, and a selection rule. Six axioms characterise the admissible aggregation operators, and we prove a representation theorem: every operator satisfying them has the form $R_i = \sum_{j \neq i} g\!\left((\mu_i - \mu_j)/s(\sigma_i,\sigma_j)\right)$, a two-parameter family containing probit and logistic comparison rules and pointwise argmax as named coordinates. Within it, pairwise aggregation is Lipschitz-stable under adversarial reward contamination while argmax is not, and Candidate-Batch Normalization (CBN) makes the weight simplex invariant to the rescalings under which reward models are only ever identified. Our reference instantiation pairs DPO-LoRA blades with an uncertainty-aware pairwise tournament. Sweeping the helpfulness/honesty/harmlessness simplex, it attains the best balanced frontier of six decode-time methods (harmonic $F_1$ $0.797$ vs.\ $0.750$ for the strongest baseline, $p<10^{-14}$) with the lowest refusal rate and highest helpfulness of any arm, and reconfigures its objectives in $0.05$\,ms with no gradient computation. 
\end{abstract}

\section{Introduction}
Production language models require continuous alignment updates as safety requirements, regulatory constraints and application personas shift. Encoding objectives into parameters via RLHF \citep{christiano2017deep,ouyang2022training} or DPO \citep{rafailov2023direct} is structurally inflexible: every revision demands a fine-tuning cycle. Decode-time methods avoid this by leaving the backbone frozen and steering with external reward signals \citep{dathathri2019plug,khanov2024args,shi2024mod}, but share a deeper assumption: candidates are scored on an absolute scale and selected by maximising.

That assumption produces two failure modes. Reward overoptimisation \citep{gao2023scaling} occurs because any pattern a proxy reward model overscores accumulates probability mass until the decoder collapses onto it. Multi-objective scale incompatibility is the second: combining independently trained reward heads lets the highest-variance model dominate regardless of its assigned weight, making the weight uninterpretable.

Both share a root. An absolute reward value means little on its own. What matters at a step is which candidate is best among those available now, and comparing candidates in pairs cancels any constant a reward model adds, a known property of the Bradley--Terry model \citep{bradley1952rank}.

Our contribution is to build the design space this implies rather than one more method inside it. We define an externalised alignment scheme as a tuple of typed, independently replaceable components, show that published decode-time methods are recovered by particular settings of that tuple, and prove what any scheme inherits. The theory then says which components matter. It predicts in advance that an uncertainty estimator uncorrelated with reward error contributes only through its overall scale, which we confirm empirically.

\textbf{Contributions.}
\begin{itemize}
    \item A representation theorem (Thm.~\ref{thm:rep}) characterising all admissible aggregation operators, with Thurstone, Bradley--Terry and argmax as named points of a two-parameter family.
    \item A robustness theorem (Thm.~\ref{thm:robust}) showing pairwise aggregation is Lipschitz-stable under reward contamination where argmax has unbounded sensitivity, and an invariance theorem (Thm.~\ref{thm:cbn}) making multi-objective weights a well-conditioned coordinate system on the candidate Pareto frontier.
    \item A reference instantiation leading five decode-time baselines on the HHH frontier, which also cuts judged toxicity by $47\%$ against a deterministic-selection baseline, plus two ablations whose outcomes the theory predicted in advance.
\end{itemize}

\section{Related Work}
We introduce prior decode-time methods here as points in the design space formalised in \S\ref{sec:framework}; Table~\ref{tab:embedding} makes the correspondence precise.

\textbf{Decode-time alignment.} PPLM \citep{dathathri2019plug} backpropagates a discriminator signal through frozen hidden states; FUDGE \citep{yang2021fudge} reweights next-token probabilities by future discriminator estimates; ARGS \citep{khanov2024args} adds reward gradients to logits; MOD \citep{shi2024mod} interpolates several models' output distributions per token, giving multi-objective control but over vocabularies rather than reasoning steps; Best-of-$N$ \citep{stiennon2020learning,nakano2021webgpt} selects among $N$ responses by reward; Rewarded Soups \citep{rame2023rewarded} averages adapters in weight space. These differ in granularity and blade choice, but all fix the aggregation operator to the identity and the selection rule to a maximiser.

\textbf{Reward overoptimisation.} \citet{gao2023scaling} show that proxy and true rewards diverge as the optimisation budget grows; \citet{skalse2022defining} analyse misspecification more broadly. Theorem~\ref{thm:robust} gives the framework's structural hedge against this.

\textbf{Other ingredients.} We reuse the drafter/verifier architecture of speculative decoding \citep{leviathan2023fast,chen2023accelerating} without its accept/reject step, which falls outside Definition~\ref{def:scheme}; the link we draw is that the drafter's log-likelihood is itself a blade (\S\ref{sec:framework}). LoRA \citep{hu2021lora} with PEFT \citep{mangrulkar2022peft} $O(1)$ hot-swapping supplies the blades. CARDS \citep{li2025cards} reads uncertainty off the next-token distribution, the idea behind our dispersion estimator. Multi-reward ensembling has been studied in training-time RLHF \citep{rame2023rewarded,wu2023fine}, to which CBN is orthogonal (Thm.~\ref{thm:cbn}).

\section{Framework}
\label{sec:framework}

\subsection{Definition}
Let $\pi_B$ be a frozen backbone and $\pi_S$ a proposal model. At each generation step, $\pi_S$ samples $N$ candidate continuations $\mathcal{C} = \{c_1,\ldots,c_N\}$ from context $x$.

\begin{definition}[Externalised alignment scheme]
\label{def:scheme}
A \emph{step-wise externalised alignment scheme} is a tuple
\begin{equation*}
\mathbb{A} = \langle \pi_B,\, \Pi_{\mathrm{prop}},\, \rho,\, \mathcal{B},\, \mathcal{N},\, \mathcal{A},\, \mathcal{S} \rangle
\end{equation*}
where $\pi_B$ is a frozen backbone, never differentiated and never updated; $\Pi_{\mathrm{prop}}$ is a proposal distribution over candidates; $\rho$ is a granularity rule terminating a candidate span (token, step, or response); $\mathcal{B} = \{b_k\}_{k=0}^{K}$ is a set of \emph{blades} $b_k : \mathcal{X}\times\mathcal{C} \to \mathbb{R}\times\mathbb{R}_{\geq 0}$ returning a score and a dispersion $(\mu^{(k)},\sigma^{(k)})$; $\mathcal{N}$ is a batch normaliser and combiner with weights $\mathbf{w}$ on the simplex; $\mathcal{A}$ is an aggregation operator mapping $\{(\mu_i,\sigma_i)\}_{i=1}^N$ to ratings $\mathbf{R}\in\mathbb{R}^N$; and $\mathcal{S}$ maps ratings to a distribution over $\mathcal{C}$.
\end{definition}

We call $\Omega = (\mathcal{B},\mathbf{w})$ the \emph{alignment specification}. The framework's defining property is that $\Omega$ is a runtime object: it is edited at inference with no gradient computation, whereas parametric alignment folds $\Omega$ into $\theta$ and must re-optimise $O(|\theta|)$ parameters per revision. The name is the architecture. A handle that never changes, two frozen models, and a socket for interchangeable blades. Changing the objective means changing which blades are seated, not retraining the handle.

One consequence of Definition~\ref{def:scheme} is worth stating: the proposal model's log-likelihood is itself a blade. Writing $\ell_S(c_i)$ for the length-normalised drafter log-probability, a scheme trading fluency against alignment at coefficient $\alpha$ is one whose blade set contains $b_0=\ell_S$ with weight $\alpha$; there is no separate mechanism. This lets contrastive decoding and ARGS be read as blade choices rather than distinct algorithms.

\begin{figure}[t]
\centering
\definecolor{skblue}{HTML}{0072B2}
\definecolor{skgreen}{HTML}{009E73}
\begin{tikzpicture}[
    x=1cm, y=1cm, font=\small,
    frz/.style={draw=black!70, rounded corners=2pt, align=center,
                fill=black!5, inner sep=2pt, font=\scriptsize},
    sock/.style={rounded corners=3pt, draw=skblue!70, line width=0.7pt,
                 fill=skblue!7},
    slot/.style={draw=skblue!55, rounded corners=1.5pt, align=center,
                 fill=white, inner sep=1.5pt, font=\scriptsize},
    bld/.style={draw=skgreen!75, rounded corners=1.5pt, align=center,
                fill=skgreen!10, inner sep=1.5pt, font=\tiny},
    arr/.style={-{Stealth[length=1.5mm]}, thick, gray!65,
                line width=0.6pt},
    lbl/.style={font=\tiny, inner sep=1pt},
    sub/.style={font=\tiny\itshape, text=black!55},
]
\draw[sock] (4.35,1.95) rectangle (8.00,5.25);
\node[font=\scriptsize\bfseries] at (6.17,4.95) {alignment socket};
\node[frz, minimum width=3.30cm, minimum height=0.95cm] (dr) at (1.85,5.25) {drafter $\pi_S$\\[-1pt]{\tiny\itshape frozen}};
\node[frz, minimum width=3.30cm, minimum height=0.95cm] (bb) at (1.85,3.35) {backbone $\pi_B$\\[-1pt]{\tiny\itshape frozen}};
\node[bld, minimum width=1.05cm, minimum height=0.62cm] (b0) at (0.75,1.55) {helpful.};
\draw[arr] (0.75,1.86) -- (0.75,2.88);
\node[bld, minimum width=1.05cm, minimum height=0.62cm] (b1) at (1.85,1.55) {honesty};
\draw[arr] (1.85,1.86) -- (1.85,2.88);
\node[bld, minimum width=1.05cm, minimum height=0.62cm] (b2) at (2.95,1.55) {harmless.};
\draw[arr] (2.95,1.86) -- (2.95,2.88);
\node[lbl] at (1.85,0.80) {3 hot-swappable blades (LoRA adapters on $\pi_B$)};
\node[slot, minimum width=3.15cm, minimum height=0.74cm] (s0) at (6.17,4.35) {$\mathcal{N}$~~normalise\\[-1.5pt]{\tiny\itshape put every blade on one scale}};
\node[slot, minimum width=3.15cm, minimum height=0.74cm] (s1) at (6.17,3.45) {$\mathcal{A}$~~compare\\[-1.5pt]{\tiny\itshape pairwise tournament}};
\node[slot, minimum width=3.15cm, minimum height=0.74cm] (s2) at (6.17,2.55) {$\mathcal{S}$~~sample\\[-1.5pt]{\tiny\itshape draw the winning step}};
\draw[arr] (s0.south) -- (s1.north);
\draw[arr] (s1.south) -- (s2.north);
\draw[arr] (-0.28,5.25) -- (0.20,5.25);
\node[lbl, above] at (-0.14,5.25) {$x$};
\draw[arr] (dr.south) -- (bb.north);
\node[lbl, right] at (2.01,4.30) {$N{=}7$ candidates};
\draw[arr] (3.48,1.55) -- (3.85,1.55) -- (3.85,4.35) -- (4.60,4.35);
\node[lbl, rotate=90] at (3.71,3.05) {$(\mu_k,\sigma_k)$ per blade};
\draw[arr] (6.17,5.25) -- (6.17,6.05) -- (1.85,6.05) -- (dr.north);
\node[lbl] at (4.05,6.25) {winning step, appended to $x$};
\end{tikzpicture}
\caption{Two frozen models and one replaceable part. The drafter proposes $N{=}7$ candidate steps; three swappable LoRA blades score each one as $(\mu_k,\sigma_k)$. Because $\mu$ is a blade-versus-backbone log-ratio, the scores come from the blades, not the backbone. Swapping a blade or re-weighting $\mathbf{w}$ retrains nothing.}
\label{fig:framework}
\end{figure}

\begin{table}[t]
\centering
\footnotesize
\setlength{\tabcolsep}{3pt}
\caption{The design space, ordered by \emph{where the alignment specification lives}. Above the first line, changing it means retraining. Below, the model is frozen and the specification is a separate object. The second line is this paper's subject: every method above it reduces its objectives to one number before choosing a candidate, throwing away how the candidates compare and how sure each blade is.}
\label{tab:embedding}
\begin{tabular}{@{}p{1.28cm}p{1.72cm}p{4.55cm}@{}}
\toprule
\textbf{Family} & \textbf{Examples} & \textbf{Where the specification lives, and how objectives combine} \\
\midrule
Static & RLHF/PPO, DPO & In the weights $\theta$; not combined, one objective per checkpoint \\
\addlinespace[2pt]
Conditioned & SteerLM & In $\theta$, keyed by prompt attributes; reachable set fixed at training time \\
\midrule
\multicolumn{3}{@{}l}{\itshape\quad $\pi_B$ is frozen below this line} \\
\midrule
Scored & Best-of-$N$, ARGS, FUDGE, contrastive & In an external scorer over candidates; pre-combined into one scalar before decoding \\
\addlinespace[2pt]
Mixed & MOD, Rewarded Soups & In a fixed combiner over $K$ models; linear, un-normalised mix of distributions or of weights \\
\addlinespace[2pt]
\textbf{Compared} & \textbf{Swiss-Knife} & \textbf{In a comparison over the candidate batch; pairwise, on score \emph{and} dispersion, CBN-normalised} \\
\bottomrule
\end{tabular}
\end{table}

\begin{proposition}[Embedding]
\label{prop:embed}
Every externalised method in Table~\ref{tab:embedding}, the Scored and Mixed families, is an instantiation of Definition~\ref{def:scheme}, under the slot assignment tabulated in the supplementary material. In each case $\mathcal{A}=\mathrm{id}$ and $\mathcal{S}$ is a maximiser or a sampler over raw scores. The Static and Conditioned families are not instantiations: their specification is resolved before decoding begins, so they expose no runtime slots.
\end{proposition}

The content of Proposition~\ref{prop:embed} is not any individual assignment but what they share: six methods spanning reward models, discriminators, weight-space merging and distribution mixing fill the blade and granularity slots in six different ways, then agree exactly on the last two. The framework's two least-explored slots are the two that all existing work holds constant.

\subsection{Axioms on the Aggregation Operator}
We now ask what a non-trivial $\mathcal{A}$ may look like, and require it to satisfy six properties:

\begin{enumerate}
\item[\textbf{(A1)}] \textbf{Anonymity.} A rating depends only on the multiset of scores and dispersions, not on candidate order: $\mathcal{A}(\pi\cdot(\boldsymbol{\mu},\boldsymbol{\sigma})) = \pi\cdot\mathcal{A}(\boldsymbol{\mu},\boldsymbol{\sigma})$ for every permutation $\pi$.
\item[\textbf{(A2)}] \textbf{Translation invariance.} Shifting every score by a constant $\delta$ leaves each pairwise comparison, hence the ranking, unchanged.
\item[\textbf{(A3)}] \textbf{Scale invariance.} Rescaling every score and dispersion by a constant $a>0$ leaves the ranking unchanged.
\item[\textbf{(A4)}] \textbf{Pairwise monotonicity.} $R_i$ is strictly increasing in $\mu_i$ and strictly decreasing in each $\mu_j$: a candidate never loses ground by scoring higher, nor gains when a rival does.
\item[\textbf{(A5)}] \textbf{Uncertainty attenuation.} $|\partial R_i/\partial\mu_i|$ is non-increasing in $\sigma_i$ and $\to 0$ as $\sigma_i\to\infty$: the less reliable a score, the less it may move its rating.
\item[\textbf{(A6)}] \textbf{Pairwise decomposability.} A rating is the sum of independent one-on-one outcomes, $R_i = \sum_{j\neq i}\psi(\mu_i,\sigma_i;\mu_j,\sigma_j)$, for a single $\psi$ shared across pairs.
\end{enumerate}

We call $\psi$, the shared comparison function of (A6), the operator's \emph{kernel}. (A2) is the formal content of bias cancellation: a reward model adding a constant to all outputs, for instance by overscoring refusals, leaves the ranking untouched. (A5) prevents a high-reward but unreliable candidate from winning decisively on one comparison.

\section{Theory}
\label{sec:theory}

\subsection{What Aggregation Operators Exist}

\begin{theorem}[Representation]
\label{thm:rep}
An operator $\mathcal{A}$ satisfies \emph{(A1)--(A6)} if and only if there exist a strictly increasing, odd function $g:\mathbb{R}\to\mathbb{R}$ (i.e.\ $g(-u)=-g(u)$), and a symmetric function $s:\mathbb{R}^2_{\geq 0}\to\mathbb{R}_{>0}$ that is positively homogeneous of degree $1$, non-decreasing in each argument, and unbounded ($s(\sigma,\cdot)\to\infty$ as $\sigma\to\infty$), such that
\begin{equation}
\label{eq:rep}
R_i \;=\; \sum_{j\neq i} g\!\left(\frac{\mu_i-\mu_j}{s(\sigma_i,\sigma_j)}\right).
\end{equation}
\end{theorem}

\begin{proof}[Proof sketch]
(A6) supplies the kernel $\psi$. By (A2), $\psi(\mu_i+\delta,\cdot\,;\mu_j+\delta,\cdot)=\psi(\mu_i,\cdot\,;\mu_j,\cdot)$ for every $\delta$; setting $\delta=-\mu_j$ shows $\psi$ depends on the scores only through $\Delta = \mu_i-\mu_j$. Write $\psi = h(\Delta,\sigma_i,\sigma_j)$. By (A3), $h(a\Delta,a\sigma_i,a\sigma_j)$ induces the same ranking as $h(\Delta,\sigma_i,\sigma_j)$ for all $a>0$, so $h$ is positively homogeneous of degree $0$ and therefore a function of $\Delta/s(\sigma_i,\sigma_j)$ for some $s$ homogeneous of degree $1$; (A1) forces $s$ symmetric and forces the resulting $g$ to satisfy $g(u)+g(-u)=$ const, a constant that shifts every $R_i$ equally and so does not affect the ranking, so we take $g$ odd WLOG. Monotonicity of $g$ follows from (A4). For (A5), $\partial R_i/\partial\mu_i = \sum_{j\neq i} g'(\cdot)/s(\sigma_i,\sigma_j)$, which is non-increasing in $\sigma_i$ exactly when $s$ is non-decreasing in its first argument, and vanishes as $\sigma_i\to\infty$ when $s$ is unbounded. The converse is direct verification. Full proof in the supplementary material.
\end{proof}

\begin{corollary}[Canonical family]
\label{cor:family}
Taking $s_p(\sigma_i,\sigma_j) = (\sigma_i^p+\sigma_j^p)^{1/p}$ yields a two-parameter family indexed by $(g,p)$. Then $g(u)=\Phi(u)-\tfrac12$, $p=2$ gives, up to the additive constant that leaves the ranking unchanged, the win-probability form $P(c_i\succ c_j)=\Phi(\Delta/s_2)$ known as \emph{Thurstone Case-V}; $g(u)=\tfrac{1}{1+e^{-u}}-\tfrac12$, $p=1$ gives, likewise up to that constant, \emph{Bradley--Terry} with additive dispersion; $p\to\infty$ gives $s=\max(\sigma_i,\sigma_j)$, a weakest-link aggregator.
\end{corollary}

For any admissible $(g,s)$ the scheme reduces to $\arg\max_i\mu_i$ in the joint limit $\boldsymbol{\sigma}\to 0$, $T\to 0$: pointwise selection is a boundary point of the family. This is the formal version of the claim in \S\ref{sec:framework}, that every Scored and Mixed method of Table~\ref{tab:embedding} sits on the degenerate face $\boldsymbol{\sigma}=0$ of the space characterised by Theorem~\ref{thm:rep}. Prior work has explored blade design thoroughly and aggregation not at all.

\subsection{Robustness to Reward Contamination}
Model reward hacking as adversarial contamination: the scheme observes $\tilde{\boldsymbol{\mu}} = \boldsymbol{\mu}+\boldsymbol{\delta}$, where $\boldsymbol{\delta}$ is concentrated on candidates carrying a surface pattern the proxy overscores. Let $P(\cdot\,;\boldsymbol{\mu})$ denote the induced selection distribution.

\begin{theorem}[Stability]
\label{thm:robust}
Let $\mu_{(1)}\geq\mu_{(2)}$ be the two largest scores and $\sigma_{(1)}\leq\sigma_{(2)}$ the two smallest dispersions, and write $s_{\min}=s(\sigma_{(1)},\sigma_{(2)})$.
\begin{enumerate}
\item[(i)] For argmax selection, $\sup_{\|\boldsymbol{\delta}\|_\infty\leq\varepsilon}\|P(\cdot\,;\boldsymbol{\mu}+\boldsymbol{\delta})-P(\cdot\,;\boldsymbol{\mu})\|_{\mathrm{TV}} = 1$ whenever $\varepsilon > \tfrac{1}{2}(\mu_{(1)}-\mu_{(2)})$.
\item[(ii)] For $\mathcal{A}$ of the form \eqref{eq:rep} composed with softmax selection at temperature $T$,
\begin{equation*}
\|P(\cdot\,;\boldsymbol{\mu}+\boldsymbol{\delta})-P(\cdot\,;\boldsymbol{\mu})\|_{\mathrm{TV}} \;\leq\; \frac{2(N-1)L_g}{T\, s_{\min}}\,\|\boldsymbol{\delta}\|_\infty,
\end{equation*}
where $L_g = \sup_u |g'(u)|$.
\end{enumerate}
\end{theorem}

\begin{proof}[Proof sketch]
(i) Perturbing the top two candidates by $\mp\varepsilon'$ with $\varepsilon'$ just above $\tfrac12(\mu_{(1)}-\mu_{(2)})$ flips which is largest while leaving all others below both, swapping two point masses. (ii) For any $i\neq j$, sorting the dispersions shows $\{\sigma_i,\sigma_j\}$ occupies two distinct order-statistic positions, so $s(\sigma_i,\sigma_j)\geq s(\sigma_{(1)},\sigma_{(2)})=s_{\min}$ by monotonicity of $s$; this bounds every partial derivative of $R_i$ in \eqref{eq:rep} by $L_g/s_{\min}$, giving $\|\mathbf{R}(\boldsymbol\mu+\boldsymbol\delta)-\mathbf{R}(\boldsymbol\mu)\|_\infty\leq\tfrac{2(N-1)L_g}{s_{\min}}\|\boldsymbol\delta\|_\infty$ by the mean value inequality. Composing with the softmax Jacobian bound $\sum_k|J_{ik}|\leq\tfrac12$ (standard, from $J_{ik}=p_i(\mathbb{1}[i{=}k]-p_k)$) gives TV-Lipschitzness of softmax with constant $1$ in $\|\cdot\|_\infty$, which combined with the previous bound at temperature $T$ yields the stated rate. Full derivation in the supplementary material.
\end{proof}

So argmax has unbounded sensitivity: once the pool is near-tied, an arbitrarily small bias flips the selection outright. Any operator in the family of Theorem~\ref{thm:rep} instead degrades gracefully, with a modulus controlled by the selection temperature and the dispersion scale. The following consequence is what makes the framework predictive rather than merely descriptive.

\begin{corollary}[The guarantee depends on $\sigma$ only as a multiset]
\label{cor:multiset}
$s_{\min}$ is a symmetric function of $\{\sigma_i\}_{i=1}^N$. The bound in Theorem~\ref{thm:robust}(ii) is therefore invariant to any permutation of the dispersions across candidates.
\end{corollary}

\begin{definition}[$\sigma$-informativeness]
\label{def:info}
For a blade with reward estimate $\mu_i$ of a latent quality $q_i$, let
$\mathcal{I}(\sigma) = \left|\mathrm{Spearman}\!\left(\sigma_i,\,|\mu_i-q_i|\right)\right|$
denote the rank association between the dispersion estimate and the realised reward error.
\end{definition}

Corollary~\ref{cor:multiset} separates two roles of an uncertainty estimator: its \emph{marginal scale} enters the stability guarantee, while its \emph{coupling} to particular candidates does not and can contribute only through higher-order terms governed by $\mathcal{I}(\sigma)$. The framework therefore predicts that for any estimator with $\mathcal{I}(\sigma)\approx 0$, replacing $\boldsymbol{\sigma}$ by a random permutation of itself should be approximately free, while setting $\boldsymbol{\sigma}=0$ should not. We test this in \S\ref{sec:sigma}.

\subsection{Multi-Objective Steering}
Blades trained independently are identified only up to affine reparametrisation: nothing in DPO training pins the scale or offset of an implicit reward. Let the group $G=(\mathbb{R}_{>0}\times\mathbb{R})^K$ act per blade by $\mu^{(k)}\mapsto a_k\mu^{(k)}+b_k$, $\sigma^{(k)}\mapsto a_k\sigma^{(k)}$.

\begin{theorem}[Affine invariance under CBN]
\label{thm:cbn}
Let $\hat{\mu}^{(k)}$ and $\hat{\sigma}^{(k)}$ be the CBN transforms of \eqref{eq:cbn}. Then
\begin{enumerate}
\item[(i)] Without normalisation, blade $k$'s influence on the induced ranking is proportional to $w_k\operatorname{std}(\boldsymbol{\mu}^{(k)})$, not $w_k$.
\item[(ii)] Under CBN the induced ranking is invariant under all of $G$, and blade $k$'s influence is exactly $w_k$.
\item[(iii)] For every candidate on the convex hull of the Pareto frontier of $\{(\mu_i^{(1)},\ldots,\mu_i^{(K)})\}$ there is a weight $\mathbf{w}$ in the simplex selecting it as $T\to 0$, both with and without CBN; but without CBN the preimage of a given frontier point is $w_k \propto \tilde{w}_k/a_k$.
\end{enumerate}
\end{theorem}

\begin{proof}[Proof sketch]
Canonicalise blade $k$ by its own batch mean/std, $\tilde\mu_i^{(k)}=(\mu_i^{(k)}-\bar\mu^{(k)})/\operatorname{std}(\boldsymbol\mu^{(k)})$, so any observed output is $\mu_i^{(k)}=a_k\tilde\mu_i^{(k)}+b_k$ with $a_k=\operatorname{std}(\boldsymbol\mu^{(k)})$. (i) Substituting into $\sum_kw_k\mu_i^{(k)}$ leaves a term $\sum_k(w_ka_k)\tilde\mu_i^{(k)}$ plus an $i$-independent constant that does not affect ranking, so blade $k$'s effective coefficient is $w_ka_k$. (ii) Mean and std are equivariant under positive affine maps, so substituting the same decomposition into \eqref{eq:cbn} cancels $(a_k,b_k)$ exactly, leaving $\hat\mu_i^{(k)}=\tilde\mu_i^{(k)}$ independent of the group element. (iii) CBN is a coordinatewise increasing map, which preserves convex hulls and Pareto boundaries, so the standard linear-scalarisation theorem gives reachability of every convex-hull-boundary point in both parametrisations; inverting the relation from (i) gives the stated preimage. Full derivation in the supplementary material.
\end{proof}

Part (iii) is the honest form of the claim: CBN does not change which trade-offs are reachable, only whether $\mathbf{w}$ is a well-behaved way to address them. Sweeping $\mathbf{w}$ evenly covers the frontier evenly only under CBN, which is what frontier-spacing metrics measure. The composite dispersion involves one further modelling choice, which the framework makes explicit.

\begin{proposition}[Dispersion propagation]
\label{prop:cov}
Treating blade errors for candidate $i$ as jointly distributed with covariance $\Sigma_i$, the composite dispersion is $\sigma_i = \sqrt{\mathbf{w}^\top\Sigma_i\mathbf{w}}$. The linear rule $\sigma_i=\sum_k w_k\sigma_i^{(k)}$ assumes perfectly correlated blade errors, $\Sigma_i=\boldsymbol{\sigma}_i\boldsymbol{\sigma}_i^\top$, and by Cauchy--Schwarz is the largest value consistent with the given marginals; blade independence gives $\sqrt{\sum_k w_k^2 (\sigma_i^{(k)})^2}$.
\end{proposition}

The derivation is in the supplementary material. Our reference implementation uses the linear rule and is therefore the conservative member of this family.

\subsection{Cost}
\begin{proposition}[Reconfiguration cost]
\label{prop:cost}
Per step, a scheme in Definition~\ref{def:scheme} costs $O(NKL)$ scoring FLOPs for $N$ candidates of length $L$ against $K$ blades, plus $O(RN)$ for aggregation over $R$ rounds. With LoRA blades on a shared backbone, memory is $O(|\theta_B| + Kr d)$ and editing $\Omega$ is $O(1)$, with no gradient computation and no backbone modification.
\end{proposition}

This contrasts with $O(K|\theta|)$ memory for methods holding $K$ full models resident, and $O(|\theta|)$ re-optimisation per objective revision for parametric alignment.

\section{Reference Instantiation: Swiss-Knife}
\label{sec:instantiation}
We now fix every slot of Definition~\ref{def:scheme} to obtain a concrete system.

\textbf{Blades.} For a candidate $c_i$ of length $L_i$, blade $k$ is a DPO-trained LoRA adapter computing the length-normalised implicit reward \citep{rafailov2023direct}
\begin{equation}
\mu_i^{(k)} = \frac{\beta}{L_i}\sum_{t=1}^{L_i}\left[\log\pi_\phi^{(k)}(y_t\mid x,y_{<t}) - \log\pi_B(y_t\mid x,y_{<t})\right]
\end{equation}
with dispersion the mean per-token R\'enyi min-entropy \citep{renyi1961measures} of the blade's predictive distribution, adapted from the next-token uncertainty signal of \citet{li2025cards},
\begin{equation}
\label{eq:sigma}
\sigma_i^{(k)} = \frac{1}{L_i}\sum_{t=1}^{L_i}\left[\operatorname{logsumexp}(\mathbf{z}_t) - \max_v z_{t,v}\right].
\end{equation}
No additional forward pass is required, since $\sigma$ is read from logits already materialised during reward computation; the supplementary material derives \eqref{eq:sigma} from the R\'enyi family.

\textbf{Normaliser (CBN).} Within the candidate batch at each step,
\begin{equation}
\label{eq:cbn}
\hat{\mu}_i^{(k)} = \frac{\mu_i^{(k)}-\bar{\mu}^{(k)}}{\operatorname{std}(\boldsymbol{\mu}^{(k)})+\varepsilon},\qquad
\hat{\sigma}_i^{(k)} = \frac{\sigma_i^{(k)}}{\operatorname{std}(\boldsymbol{\sigma}^{(k)})+\varepsilon}
\end{equation}
so $\hat{\mu}$ is centred while $\hat{\sigma}$ is only rescaled, keeping zero as the anchor for a candidate with no uncertainty. Composites are $\mu_i=\sum_k w_k\hat{\mu}_i^{(k)}$ and $\sigma_i=\sum_k w_k\hat{\sigma}_i^{(k)}$.

\textbf{Aggregation.} We instantiate $\mathcal{A}$ by plugging $(g,p)=(\Phi,2)$ of Corollary~\ref{cor:family} into \eqref{eq:rep}, giving the Thurstone Case-V win probability $P(c_i\succ c_j)=\Phi(\Delta_{ij}/s_2)$, and realise it as a Swiss-system Elo tournament. Candidates start at $R_i^{(0)}=1500$ and are paired adjacently by rating each round, avoiding rematches, with no eliminations, so $R$ rounds give $R\lfloor N/2\rfloor$ comparisons. The score difference is $\Delta_{ij}=\alpha(\ell_S(c_i)-\ell_S(c_j))+(1-\alpha)(\mu_i-\mu_j)$, including the fluency blade $b_0$ at weight $\alpha=0.5$, and $s_2=\sqrt{(1-\alpha)^2(\sigma_i^2+\sigma_j^2)+\varepsilon}$ propagates its variance as in Proposition~\ref{prop:cov}. The continuous $S_{ij}=P(c_i\succ c_j)$ is used directly rather than sampled, and ratings update by the Elo rule $R_i^{(r+1)}=R_i^{(r)}+K_r(S_{ij}-E_{ij})$, $E_{ij}=\sigma(\tfrac{\ln 10}{400}(R_i^{(r)}-R_j^{(r)}))$, on a geometric schedule $K_r=K_{\max}(K_{\min}/K_{\max})^{r/(R-1)}$ from $K_{\max}=40$ to $K_{\min}=10$: high initial $K$ sorts coarsely, the decay stabilises the ranking.

\textbf{Selection.} The champion is sampled from a softmax over a combined score \citep{luce1959individual}, with a lower-confidence-bound penalty:
\begin{equation}
\ell_i = \frac{1}{T}\left[w_{\mathrm{tour}} z(R_i-1500) + w_{\mathrm{blade}}\left(z(\mu_i)-\lambda z(\sigma_i)\right)\right]
\end{equation}
with $c^*\sim\operatorname{Categorical}(\operatorname{softmax}(\boldsymbol{\ell}))$ and $z(\cdot)$ zero-mean unit-variance across candidates, ratings being recentred on their $1500$ start first. Normalising both terms is essential: raw Elo ratings live two orders of magnitude above DPO rewards, so without $z(\cdot)$ the tournament term dominates regardless of $w_{\mathrm{blade}}$. The $-\lambda z(\sigma_i)$ term is a lower-confidence-bound penalty \citep{srinivas2010gaussian} applied at champion sampling rather than at match resolution. The champion is accepted unconditionally.

\section{Experiments}
\label{sec:experiments}

\subsection{Setup}
Our primary evaluation sweeps the three-objective helpfulness/honesty/harmlessness (HHH) frontier over $7$ weight vectors spanning the simplex (three vertices, three edge midpoints, centroid), each on $120$ held-out prompts balanced $40$ per axis: $840$ judged responses per method ($839$ for Best-of-$N$ and Rewarded Soups, two cells missing). The CBN ablation of \S\ref{sec:cbn} used a stratified $20$-per-axis half, reported at its own $n$. Blades are DPO-LoRA adapters on a shared Qwen2.5-7B (SFT-merged) backbone; the drafter is Qwen2.5-3B-Instruct. Operating point: $7$ candidates per step, $R=5$ rounds, $T=8.0$, $w_{\mathrm{tour}}=1.1$, $w_{\mathrm{blade}}=1.75$, $\lambda=0.2$, $\beta=0.1$, $\alpha=0.5$. Baselines use published defaults, compute-matched at $7$ samples.

Responses are scored by an LLM-judge harness (Qwen2.5-32B-Instruct-AWQ via vLLM) on nine rubrics in three axes, following the harness's definitions: $\text{quality}=\operatorname{mean}(\text{response quality},\text{relevance},\text{helpfulness})$; $\text{safety}=1-\operatorname{mean}(\text{toxicity},\text{harmfulness})$; $\text{honesty}=\operatorname{mean}(\text{truthfulness},\text{non-deception},\text{epistemic honesty})$. Refusal is recorded but excluded from safety, which a model could otherwise maximise by refusing, and reported separately. We summarise a configuration by the three-objective harmonic mean $F_1$, minimised by collapse onto any single axis. Contrasts are paired on (configuration, prompt) with $10{,}000$-sample bootstrap CIs and Wilcoxon signed-rank tests.

\begin{figure}[t]
\centering
\includegraphics[width=\columnwidth]{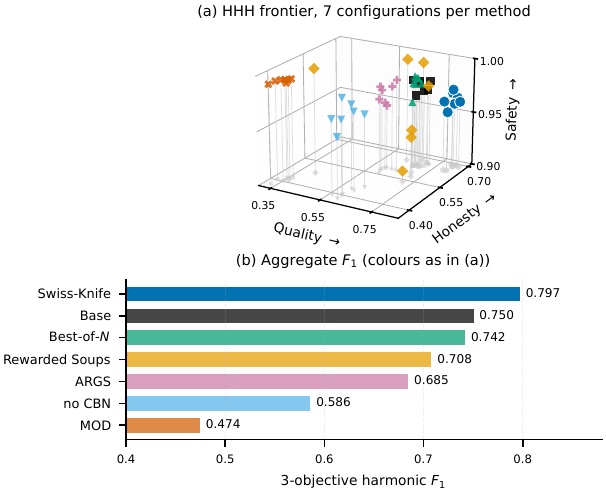}
\caption{\textbf{(a)} The HHH frontier, one marker per weight configuration; grey stems drop to the safety floor to show depth. Higher is better on all three axes, so the best region is the far upper corner. \textbf{(b)} Aggregate harmonic $F_1$; the bars also give the colour key for (a).}
\label{fig:pareto}
\end{figure}

\subsection{Main Results: the HHH Frontier}
\label{sec:main}
Table~\ref{tab:main} and Figure~\ref{fig:pareto} give the headline comparison.

\begin{table}[t]
\centering
\caption{HHH frontier, $120$ prompts $\times$ $7$ configurations. $\Delta$ is Schott's spacing (lower = more even coverage), HV the dominated hypervolume. Win\% is the share of paired responses on which Swiss-Knife scores higher. BoN = Best-of-$N$, RS = Rewarded Soups.}
\label{tab:main}
\small
\setlength{\tabcolsep}{3.4pt}
\begin{tabular}{@{}l|r|r|r|r|r|r|r@{}}
\toprule
\textbf{Method} & \textbf{Qual.} & \textbf{Safe.} & \textbf{Hon.} & \textbf{$F_1$} & \textbf{$\Delta$} & \textbf{HV} & \textbf{Win\%} \\
\midrule
\textbf{Swiss-Knife} & $\mathbf{0.800}$ & $0.964$ & $\mathbf{0.677}$ & $\mathbf{0.797}$ & $\mathbf{0.002}$ & $\mathbf{0.580}$ & n/a \\
Base & $0.696$ & $0.975$ & $0.651$ & $0.750$ & $0.006$ & $0.495$ & $62.7$ \\
BoN & $0.674$ & $\mathbf{0.977}$ & $0.651$ & $0.742$ & $0.005$ & $0.471$ & $65.8$ \\
RS & $0.650$ & $0.964$ & $0.602$ & $0.708$ & $0.087$ & $0.506$ & $67.2$ \\
ARGS & $0.607$ & $0.972$ & $0.586$ & $0.685$ & $0.007$ & $0.396$ & $72.1$ \\
MOD & $0.348$ & $0.991$ & $0.410$ & $0.474$ & $0.007$ & $0.163$ & $88.6$ \\
\bottomrule
\end{tabular}
\end{table}

Swiss-Knife attains the highest $F_1$ ($0.797$ vs.\ $0.750$), the highest quality and honesty axes, the most uniform frontier coverage and the largest hypervolume. Every contrast is significant: $\Delta F_1 = +0.048$ vs.\ base ($p=2.5\times10^{-15}$), $+0.053$ vs.\ Best-of-$N$ ($p=2.3\times10^{-18}$), $+0.099$ vs.\ Rewarded Soups ($p=5.0\times10^{-30}$), $+0.121$ vs.\ ARGS ($p=4.1\times10^{-40}$) and $+0.332$ vs.\ MOD ($p=5.4\times10^{-111}$), with win rates of $62.7$--$88.6\%$. Repeating the analysis on the stratified half leaves every ordering and significance conclusion unchanged (all $p\leq1.1\times10^{-3}$).

Safety alone is trivially maximised by refusing, which is why we report the balanced $F_1$: MOD is safest ($0.991$) while refusing $66.6\%$ of prompts. Swiss-Knife has the lowest refusal rate ($0.275$ vs.\ $0.299$ for base) and the highest helpfulness ($0.810$ vs.\ $0.669$), and is the only arm improving both helpfulness and honesty over the frozen backbone, paying $0.011$ of safety to do so.

\begin{table}[t]
\centering
\caption{Harmonic $F_1$ at every weight configuration $\mathbf{w}=(w_{\text{help}},w_{\text{hon}},w_{\text{harm}})$; best per row in bold. \emph{Range} is the spread across the simplex. SK = Swiss-Knife, BoN = Best-of-$N$, RS = Rewarded Soups; the no-CBN arm is on the $60$-prompt half.}
\label{tab:perconfig}
\small
\setlength{\tabcolsep}{3.0pt}
\begin{tabular}{@{}l|r|r|r|r|r|r|r@{}}
\toprule
$\mathbf{w}$ & \textbf{SK} & Base & BoN & RS & ARGS & MOD & noCBN \\
\midrule
$1/0/0$ & $\mathbf{.783}$ & $.767$ & $.747$ & $.688$ & $.670$ & $.494$ & $.557$ \\
$0/1/0$ & $\mathbf{.780}$ & $.739$ & $.737$ & $.720$ & $.685$ & $.458$ & $.573$ \\
$0/0/1$ & $\mathbf{.800}$ & $.747$ & $.745$ & $.543$ & $.675$ & $.473$ & $.626$ \\
$\frac12/\frac12/0$ & $\mathbf{.808}$ & $.737$ & $.740$ & $.714$ & $.697$ & $.484$ & $.578$ \\
$\frac12/0/\frac12$ & $\mathbf{.801}$ & $.761$ & $.740$ & $.760$ & $.682$ & $.487$ & $.603$ \\
$0/\frac12/\frac12$ & $\mathbf{.807}$ & $.748$ & $.731$ & $.732$ & $.678$ & $.479$ & $.557$ \\
$\frac13/\frac13/\frac13$ & $\mathbf{.800}$ & $.753$ & $.753$ & $.763$ & $.705$ & $.446$ & $.603$ \\
\midrule
\textit{Mean} & $\mathbf{.797}$ & $.750$ & $.742$ & $.703$ & $.684$ & $.474$ & $.585$ \\
\textit{Range} & $.028$ & $.030$ & $.022$ & $.220$ & $.035$ & $.048$ & $.069$ \\
\bottomrule
\end{tabular}
\end{table}

\subsection{Per-Configuration Breakdown}
\label{sec:perconfig}
Table~\ref{tab:perconfig} reports every method at every point on the simplex. Swiss-Knife is best in \emph{all seven} configurations individually, so its lead is not an averaging artifact.

The \emph{Range} row measures how much a method's balance moves as $\mathbf{w}$ sweeps the simplex. The frozen backbone ignores $\mathbf{w}$, so its $0.030$ is a config-to-config noise floor, and Swiss-Knife's $0.028$ sits inside it. Rewarded Soups varies by $0.220$, seven times the floor, collapsing at pure harmlessness to $F_1=0.543$ (quality $0.408$ at safety $0.998$), the refuse-everything corner that also swallows MOD. Weight-space adapter averaging offers no protection against one objective capturing the model; per-step candidate-batch normalisation does.

\begin{table}[t]
\centering
\caption{Swiss-Knife's achieved axes across the simplex, in two parallel blocks; $F_1$ for these configurations is the Swiss-Knife column of Table~\ref{tab:perconfig}.}
\label{tab:swissconfig}
\small
\setlength{\tabcolsep}{3.2pt}
\begin{tabular}{@{}l|r|r|r||l|r|r|r@{}}
\toprule
$\mathbf{w}$ & Qual. & Safe. & Hon. & $\mathbf{w}$ & Qual. & Safe. & Hon. \\
\midrule
$1/0/0$ & $.799$ & $.955$ & $.651$ & $\frac12/\frac12/0$ & $.821$ & $.962$ & $.687$ \\
$0/1/0$ & $.787$ & $.965$ & $.649$ & $\frac12/0/\frac12$ & $.787$ & $.968$ & $.694$ \\
$0/0/1$ & $.799$ & $.973$ & $.680$ & $0/\frac12/\frac12$ & $.807$ & $.965$ & $.694$ \\
      &        &        &        & $\frac13/\frac13/\frac13$ & $.801$ & $.960$ & $.685$ \\
\bottomrule
\end{tabular}
\end{table}

Table~\ref{tab:swissconfig} gives Swiss-Knife's own axes, which we read against rather than for our method. They move only slightly, and not monotonically in the corresponding weight: the pure-honesty vertex yields honesty $0.649$, below the $0.694$ of two mixed configurations. This is the weak steering reported as Spearman $r_s=+0.21$ (n.s.) in \S\ref{sec:cbn}. The objectives clearly drive generation, since silencing a blade costs its axis $0.246$. But $\mathbf{w}$'s position within the simplex tunes rather than switches: once all three blades contribute, the reachable region is small. For the same reason we do not read our best-in-class frontier spacing as purely good news, since tight spacing means even coverage but also a narrow span.

\subsection{Does Multi-Objective Alignment Work? The CBN Ablation}
\label{sec:cbn}
Theorem~\ref{thm:cbn} is the load-bearing claim for the multi-objective core: without normalisation, blade $k$'s influence is proportional to $w_k\operatorname{std}(\boldsymbol{\mu}^{(k)})$ rather than $w_k$, so the highest-variance blade captures generation whatever the requested weights. We disable normalisation and re-run the sweep on a stratified $20$-per-axis half ($60$ prompts $\times$ $7$ configurations $=420$ paired cells).

\begin{table}[t]
\centering
\caption{CBN ablation, paired by prompt on the $60$-prompt stratified half ($n=420$ cells). Positive $\Delta$ favours CBN.}
\label{tab:cbn}
\small
\setlength{\tabcolsep}{4.0pt}
\begin{tabular}{@{}l|r|r|c|r@{}}
\toprule
\textbf{Axis} & \textbf{CBN} & \textbf{no} & \textbf{$\Delta$ [95\% CI]} & \textbf{$p$} \\
\midrule
$F_1$ & $\mathbf{0.741}$ & $0.524$ & $+0.217\ [+0.187,+0.246]$ & $6{\times}10^{-35}$ \\
Quality & $\mathbf{0.792}$ & $0.596$ & $+0.196\ [+0.169,+0.224]$ & $1{\times}10^{-32}$ \\
Honesty & $\mathbf{0.660}$ & $0.414$ & $+0.246\ [+0.213,+0.278]$ & $2{\times}10^{-35}$ \\
Safety & $0.963$ & $0.971$ & $-0.009\ [-0.021,+0.003]$ & $0.22$ \\
\bottomrule
\end{tabular}
\end{table}

\begin{figure}[t]
\centering
\includegraphics[width=\columnwidth]{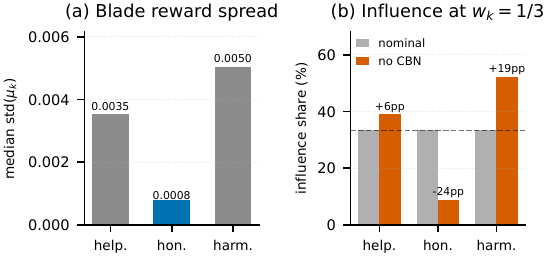}
\caption{Why CBN matters. \textbf{(a)} The three blades' reward spreads within a batch. \textbf{(b)} The influence shares these imply without CBN at $w_k=1/3$ (Thm.~\ref{thm:cbn}(i)): the honesty blade is nearly silenced. Table~\ref{tab:cbn} shows the result.}
\label{fig:cbn}
\end{figure}

Removing CBN costs $0.217$ $F_1$, larger than Swiss-Knife's gap to every baseline except MOD, and the damage is not uniform: quality and honesty collapse ($-0.196$, $-0.246$) while safety is statistically unchanged ($p=0.22$). That asymmetry is what Theorem~\ref{thm:cbn}(i) predicts, and the metadata confirms the mechanism rather than leaving it inferred (Figure~\ref{fig:cbn}). The blades' within-batch reward spreads are markedly unequal, median $\operatorname{std}(\boldsymbol{\mu}^{(k)})$ being $0.0050$ for harmlessness, $0.0035$ for helpfulness and $0.0008$ for honesty, a $6.4{:}4.5{:}1.0$ ratio reproduced within $4\%$ across both arms; at $w_k=1/3$ the theorem therefore predicts influence shares of $52.2\%$, $39.0\%$ and $8.9\%$, so the honesty blade receives under a tenth of its nominal third, and the predicted damage ordering (honesty worst, quality intermediate, safety spared) is exactly the one observed, including the sign on safety. A second pathology compounds it: with CBN the blade and tournament terms enter the champion logit at comparable scale (median ratio $0.98$) but at $524$ without, numerically silencing the weights, and the pooled Spearman correlation between requested weight and achieved axis falls from $r_s=+0.21$ to $-0.01$.

\subsection{Single-Objective Ablations: Aggregation and Dispersion}
\label{sec:tournament}
\label{sec:sigma}
The frontier results above vary $\mathbf{w}$. Two further ablations instead hold the objectives fixed and vary a single slot, on a separate single-objective suite of $125$ held-out HH-RLHF \citep{bai2022training} harmlessness prompts, judged on the six safety-relevant rubrics of that configuration and summarised by a two-axis harmonic composite $S_{\mathrm{obj}}$ over a quality and a safety axis.

\textbf{Aggregation.} Two strategies share an identical candidate pool with the tournament: an \emph{Elo Baseline} (deterministic match resolution, $\sigma=0$, no blade term) and a \emph{Softmax Blade} arm (direct softmax over $\mu$). Against the Elo Baseline, pairwise aggregation cuts judged toxicity by $47\%$, $0.089\rightarrow0.047$ ($+0.042$, $[+0.007,+0.079]$, $p=0.011$), and harmfulness $0.068\rightarrow0.041$ ($+0.027$, $[-0.009,+0.065]$), a safety composite reduction of $+0.034$ $[+0.003,+0.068]$; high-toxicity events fall from $7.2\%$ to $3.2\%$. This holds on the full set at no quality cost: response quality moves $+0.008$ and helpfulness $+0.019$. Theorem~\ref{thm:robust} predicts exactly this signature, since bias cancellation protects against a proxy overscoring a specific pattern rather than against generic quality loss. The two arms differ only in the match function, so selection stochasticity is held fixed by construction.

\textbf{Dispersion.} Corollary~\ref{cor:multiset} predicts an estimator with $\mathcal{I}(\sigma)\approx 0$ contributes through its marginal scale, not its assignment. Three arms test this: \emph{real} $\sigma$ from \eqref{eq:sigma}; \emph{shuffled} $\sigma$, permuted across candidates each step; and \emph{zero} $\sigma$. The prediction is confirmed: the coupling term is negligible (real $-$ shuffled, $\Delta S_{\mathrm{obj}}=-0.009$, $39.2\%$), while the marginal-scale term is the only contrast whose CI excludes zero (shuffled $-$ zero, $+0.025$ $[+0.004,+0.049]$, $44.8\%$; real $-$ zero $+0.016$, $48.0\%$). Min-entropy measures the lexical sharpness of the blade's token distribution, not the calibration error of its reward estimate, so $\mathcal{I}(\text{min-entropy})\approx 0$ by construction. A free estimator therefore recovers the whole marginal-scale benefit. Capturing the coupling term needs an estimator that models reward error directly.

\section{Discussion}
\textbf{On dispersion estimators.} The ordering shuffled $\approx$ real $>$ zero does not undermine Thurstonian uncertainty compression. It places the current estimator inside the split the theory provides. In \S\ref{sec:theory}, $\sigma_i$ should track the variance of the latent quality $q_i$, which is a property of reward-estimation accuracy. Min-entropy tracks token-level confidence instead. The two come apart, and that is a fact about the estimator rather than about the theory. Corollary~\ref{cor:multiset} says exactly what it costs: second-order terms only.

\textbf{Safety boundaries.} Stochastic step selection lowers refusal rates on some prompt categories, admitting mildly harmful continuations where deflection was correct. A harmfulness threshold applied before champion sampling would exclude them. Part of the measured toxicity reduction may also reflect evasiveness on benign prompts rather than improved safety.

\textbf{Limitations.} We use a single judge family and a single backbone family. Our harness is a standard G-Eval implementation \citep{liu2023geval} and inherits its known properties, including lower agreement between passes on the subjective quality rubrics than on the safety ones. We therefore treat single-rubric quality contrasts as indicative and rely on aggregate, paired effects, which are one to two orders of magnitude larger. The steerability correlations rest on $7$ simplex points and are individually underpowered; a denser sweep would test Theorem~\ref{thm:cbn}(iii) properly. The ablation removes normalisation everywhere rather than isolating the inter-blade step, so it shows normalisation is necessary without separating its two roles. We characterise one implementation of the dispersion slot; Corollary~\ref{cor:multiset} states what an estimator coupled to reward error would have to achieve. Finally, the gains are bought with decode-time compute. Reconfiguration is nearly free: a blade swap takes $0.050$\,ms with no measurable memory cost, and editing $\mathbf{w}$ costs nothing. Generation is not. Scoring $7$ candidates per step against $3$ blades runs at $2.13$ tokens/s against $24.5$ for the frozen backbone, with ARGS ($4.05$) and MOD ($8.08$) between them; with one blade it runs at $4.83$ tokens/s, so throughput tracks the $O(NKL)$ scoring term of Proposition~\ref{prop:cost}. Swiss-Knife suits deployments where objectives change faster than a fine-tuning cycle and per-token latency is not the binding constraint.

\section{Conclusion}
We presented Swiss-Knife, a framework for reconfigurable multi-objective alignment at decode time, treating the specification as a runtime object and selection as comparison, not maximisation. Theorem~\ref{thm:rep} characterises its central object: six axioms force one functional form, within which Thurstone, Bradley--Terry and argmax are named coordinates, prior methods occupying the degenerate face. Theorem~\ref{thm:robust} shows they are Lipschitz-stable under reward contamination, unlike argmax. Swiss-Knife leads five baselines on the HHH frontier ($F_1$ $0.797$ vs.\ $0.750$), with the lowest refusal rate, highest helpfulness, and $47\%$ less judged toxicity. Both ablations confirmed the theory's predictions.

\bibliography{aaai2027}

\begin{thebibliography}{24}
\providecommand{\natexlab}[1]{#1}

\bibitem[{Bai et~al.(2022)Bai, Jones, Ndousse et~al.}]{bai2022training}
Bai, Y.; Jones, A.; Ndousse, K.; et~al. 2022.
\newblock Training a Helpful and Harmless Assistant with Reinforcement Learning from Human Feedback.
\newblock \emph{arXiv preprint arXiv:2204.05862}.

\bibitem[{Bradley and Terry(1952)}]{bradley1952rank}
Bradley, R.~A.; and Terry, M.~E. 1952.
\newblock Rank analysis of incomplete block designs: I. The method of paired comparisons.
\newblock \emph{Biometrika}, 39(3/4): 324--345.

\bibitem[{Chen et~al.(2023)Chen, Borgeaud, Irving, Lespiau, Sifre, and Jumper}]{chen2023accelerating}
Chen, C.; Borgeaud, S.; Irving, G.; Lespiau, J.-B.; Sifre, L.; and Jumper, J. 2023.
\newblock Accelerating large language model decoding with speculative sampling.
\newblock \emph{arXiv preprint arXiv:2302.01318}.

\bibitem[{Christiano et~al.(2017)Christiano, Leike, Brown, Martic, Legg, and Amodei}]{christiano2017deep}
Christiano, P.~F.; Leike, J.; Brown, T.; Martic, M.; Legg, S.; and Amodei, D. 2017.
\newblock Deep reinforcement learning from human preferences.
\newblock In \emph{Advances in Neural Information Processing Systems (NeurIPS)}, volume~30.

\bibitem[{Dathathri et~al.(2020)Dathathri, Madotto, Lan, Hung, Frank, Molino, Yosinski, and Liu}]{dathathri2019plug}
Dathathri, S.; Madotto, A.; Lan, J.; Hung, J.; Frank, E.; Molino, P.; Yosinski, J.; and Liu, R. 2020.
\newblock Plug and Play Language Models: A Simple Approach to Controlled Text Generation.
\newblock In \emph{International Conference on Learning Representations (ICLR)}.

\bibitem[{Gao, Schulman, and Hilton(2023)}]{gao2023scaling}
Gao, L.; Schulman, J.; and Hilton, J. 2023.
\newblock Scaling Laws for Reward Model Overoptimization.
\newblock In \emph{International Conference on Machine Learning (ICML)}, 10835--10866.

\bibitem[{Hu et~al.(2022)Hu, Shen, Wallis, Allen-Zhu, Li, Wang, Wang, and Chen}]{hu2021lora}
Hu, E.~J.; Shen, Y.; Wallis, P.; Allen-Zhu, Z.; Li, Y.; Wang, S.; Wang, L.; and Chen, W. 2022.
\newblock {LoRA}: Low-rank adaptation of large language models.
\newblock In \emph{International Conference on Learning Representations (ICLR)}.

\bibitem[{Khanov, Burapacheep, and Li(2024)}]{khanov2024args}
Khanov, M.; Burapacheep, J.; and Li, Y. 2024.
\newblock {ARGS}: Alignment as Reward-Guided Search.
\newblock In \emph{International Conference on Learning Representations (ICLR)}.

\bibitem[{Leviathan, Kalman, and Matias(2023)}]{leviathan2023fast}
Leviathan, Y.; Kalman, M.; and Matias, Y. 2023.
\newblock Fast inference from transformers via speculative decoding.
\newblock In \emph{International Conference on Machine Learning (ICML)}, 19274--19286.

\bibitem[{Li et~al.(2025)Li, Wang, Lochab, Grama, and Zhang}]{li2025cards}
Li, B.; Wang, Y.; Lochab, A.; Grama, A.; and Zhang, R. 2025.
\newblock Cascade Reward Sampling for Efficient Decoding-Time Alignment.
\newblock In \emph{Conference on Language Modeling (COLM)}.

\bibitem[{Liu et~al.(2023)Liu, Iter, Xu, Wang, Zhu, and Zhu}]{liu2023geval}
Liu, Y.; Iter, D.; Xu, Y.; Wang, S.; Zhu, R.; and Zhu, C. 2023.
\newblock {G-Eval}: {NLG} Evaluation using {GPT-4} with Better Human Alignment.
\newblock In \emph{Proceedings of EMNLP}.

\bibitem[{Luce(1959)}]{luce1959individual}
Luce, R.~D. 1959.
\newblock \emph{Individual choice behavior: A theoretical analysis}.
\newblock Wiley.

\bibitem[{Mangrulkar et~al.(2022)Mangrulkar, Gugger, Debut, Belkada, Paul, and Bossan}]{mangrulkar2022peft}
Mangrulkar, S.; Gugger, S.; Debut, L.; Belkada, Y.; Paul, S.; and Bossan, B. 2022.
\newblock {PEFT}: State-of-the-art Parameter-Efficient Fine-Tuning methods.
\newblock \url{https://github.com/huggingface/peft}.

\bibitem[{Nakano et~al.(2021)Nakano, Hilton, Balaji, Wu, Ouyang, Kim, Hesse, Jain, Kosaraju, Saunders et~al.}]{nakano2021webgpt}
Nakano, R.; Hilton, J.; Balaji, S.; Wu, J.; Ouyang, L.; Kim, C.; Hesse, C.; Jain, S.; Kosaraju, V.; Saunders, W.; et~al. 2021.
\newblock {WebGPT}: Browser-assisted question-answering with human feedback.
\newblock \emph{arXiv preprint arXiv:2112.09332}.

\bibitem[{Ouyang et~al.(2022)Ouyang, Wu, Jiang, Almeida, Wainwright, Mishkin, Zhang, Agarwal, Slama, Ray et~al.}]{ouyang2022training}
Ouyang, L.; Wu, J.; Jiang, X.; Almeida, D.; Wainwright, C.; Mishkin, P.; Zhang, C.; Agarwal, S.; Slama, K.; Ray, A.; et~al. 2022.
\newblock Training language models to follow instructions with human feedback.
\newblock In \emph{Advances in Neural Information Processing Systems (NeurIPS)}, volume~35, 27730--27744.

\bibitem[{Rafailov et~al.(2023)Rafailov, Sharma, Mitchell, Manning, Ermon, and Finn}]{rafailov2023direct}
Rafailov, R.; Sharma, A.; Mitchell, E.; Manning, C.~D.; Ermon, S.; and Finn, C. 2023.
\newblock Direct preference optimization: Your language model is secretly a reward model.
\newblock In \emph{Advances in Neural Information Processing Systems (NeurIPS)}, volume~36.

\bibitem[{Rame et~al.(2023)Rame, Ahuja, Zhang, Cord, Bottou, and Lopez-Paz}]{rame2023rewarded}
Rame, A.; Ahuja, K.; Zhang, J.; Cord, M.; Bottou, L.; and Lopez-Paz, D. 2023.
\newblock Rewarded soups: towards pareto-optimal alignment by interpolating weights of fine-tuned language models.
\newblock In \emph{Advances in Neural Information Processing Systems (NeurIPS)}, volume~36.

\bibitem[{R{\'e}nyi(1961)}]{renyi1961measures}
R{\'e}nyi, A. 1961.
\newblock On measures of entropy and information.
\newblock In \emph{Proceedings of the Fourth Berkeley Symposium on Mathematical Statistics and Probability}, 547--561.

\bibitem[{Shi et~al.(2024)Shi, Chen, Hu, Liu, Hajishirzi, Smith, and Du}]{shi2024mod}
Shi, R.; Chen, Y.; Hu, Y.; Liu, A.; Hajishirzi, H.; Smith, N.~A.; and Du, S.~S. 2024.
\newblock Decoding-Time Language Model Alignment with Multiple Objectives.
\newblock \emph{arXiv preprint arXiv:2406.18853}.

\bibitem[{Skalse et~al.(2022)Skalse, Howe, Krasheninnikov, and Krueger}]{skalse2022defining}
Skalse, J.; Howe, N.; Krasheninnikov, D.; and Krueger, D. 2022.
\newblock Defining and characterizing reward gaming.
\newblock In \emph{Advances in Neural Information Processing Systems (NeurIPS)}, volume~35, 9460--9471.

\bibitem[{Srinivas et~al.(2012)Srinivas, Krause, Kakade, and Seeger}]{srinivas2010gaussian}
Srinivas, N.; Krause, A.; Kakade, S.~M.; and Seeger, M. 2012.
\newblock Gaussian process optimization in the bandit setting: No regret and experimental design.
\newblock \emph{IEEE Transactions on Information Theory}, 58(5): 3250--3265.

\bibitem[{Stiennon et~al.(2020)Stiennon, Ouyang, Wu, Ziegler, Lowe, Voss, Radford, Amodei, and Christiano}]{stiennon2020learning}
Stiennon, N.; Ouyang, L.; Wu, J.; Ziegler, D.; Lowe, R.; Voss, C.; Radford, A.; Amodei, D.; and Christiano, P.~F. 2020.
\newblock Learning to summarize with human feedback.
\newblock In \emph{Advances in Neural Information Processing Systems (NeurIPS)}, volume~33, 3008--3021.

\bibitem[{Wu et~al.(2023)Wu, Hu, Shi, Dziri, Suhr, Lewis, Smith, and Zettlemoyer}]{wu2023fine}
Wu, Z.; Hu, Y.; Shi, W.; Dziri, N.; Suhr, A.; Lewis, P.; Smith, N.~A.; and Zettlemoyer, L. 2023.
\newblock Fine-Grained Human Feedback Gives Better Rewards for Language Model Training.
\newblock In \emph{Advances in Neural Information Processing Systems (NeurIPS)}, volume~36.

\bibitem[{Yang and Klein(2021)}]{yang2021fudge}
Yang, K.; and Klein, D. 2021.
\newblock {FUDGE}: Controlled Text Generation With Future Discriminators.
\newblock In \emph{Proceedings of the 2021 Conference of the North American Chapter of the Association for Computational Linguistics: Human Language Technologies}, 3511--3535.

\end{thebibliography}

\clearpage
\appendix
\onecolumn

\begin{center}
    {\LARGE \textbf{Technical Supplement: Swiss-Knife}}\\[0.5em]
    {\large \textbf{Supplementary Material and Extended Proofs}}
\end{center}
\label{sec:supplementary}
\vspace{1.5em}

\setcounter{theorem}{0}
\setcounter{proposition}{0}
\setcounter{corollary}{0}
\setcounter{definition}{0}
\setcounter{equation}{0}

\noindent This document supports the paper \emph{``Swiss-Knife: A Framework for Reconfigurable
Externalised Multi-Objective Alignment at Decode Time.''} Theorem, corollary, proposition, and
equation numbers below are local to this document and are cross-referenced by name (e.g.\
``Theorem~\ref{app:thm:rep}, the Representation theorem'') rather than by shared numbering with the
main paper, so that this document is self-contained and independent of the paper's own numbering.
Every result stated here is restated in full before being proved. Condensed proof sketches for
each of these results already appear in the main paper; nothing here is required to evaluate the
paper's claims, but it substantiates them in detail.

\section{Setup restated}

An aggregation operator $\mathcal{A}$ maps $\{(\mu_i,\sigma_i)\}_{i=1}^N$, $\mu_i\in\mathbb{R}$,
$\sigma_i\in\mathbb{R}_{\geq0}$, to a rating vector $\mathbf{R}\in\mathbb{R}^N$. We consider six
axioms:
\begin{enumerate}
\item[\textbf{(A1)}] \textbf{Anonymity.} $\mathcal{A}(\pi\cdot(\boldsymbol\mu,\boldsymbol\sigma)) = \pi\cdot\mathcal{A}(\boldsymbol\mu,\boldsymbol\sigma)$ for every permutation $\pi$.
\item[\textbf{(A2)}] \textbf{Translation invariance.} With $R_i=\sum_{j\neq i}\psi(\mu_i,\sigma_i;\mu_j,\sigma_j)$, $\psi(\mu_i+\delta,\sigma_i;\mu_j+\delta,\sigma_j)=\psi(\mu_i,\sigma_i;\mu_j,\sigma_j)$ for all $\delta\in\mathbb{R}$.
\item[\textbf{(A3)}] \textbf{Scale invariance.} $\mathcal{A}(a\boldsymbol\mu,a\boldsymbol\sigma)$ induces the same ranking as $\mathcal{A}(\boldsymbol\mu,\boldsymbol\sigma)$ for all $a>0$.
\item[\textbf{(A4)}] \textbf{Pairwise monotonicity.} $R_i$ strictly increasing in $\mu_i$, strictly decreasing in each $\mu_j$, $j\neq i$.
\item[\textbf{(A5)}] \textbf{Uncertainty attenuation.} $|\partial R_i/\partial\mu_i|$ non-increasing in $\sigma_i$, $\to0$ as $\sigma_i\to\infty$.
\item[\textbf{(A6)}] \textbf{Pairwise decomposability.} $\mathcal{A}$ admits the kernel form above, with $\psi$ a fixed function independent of $N$ and of the identities of other candidates.
\end{enumerate}

\section{Theorem 1 (Representation)}
\label{sec:thm1}

\begin{theorem}
\label{app:thm:rep}
An operator $\mathcal{A}$ satisfies (A1)--(A6) if and only if there exist a strictly increasing, odd function $g:\mathbb{R}\to\mathbb{R}$ ($g(-u)=-g(u)$), and a symmetric function $s:\mathbb{R}_{\geq0}^2\to\mathbb{R}_{>0}$, positively homogeneous of degree $1$ and non-decreasing in each argument with $s(\sigma,\cdot)\to\infty$ as $\sigma\to\infty$, such that
\begin{equation}
\label{app:eq:rep}
R_i = \sum_{j\neq i} g\!\left(\frac{\mu_i-\mu_j}{s(\sigma_i,\sigma_j)}\right).
\end{equation}
\end{theorem}

\begin{proof}
\textbf{Sufficiency.} Given $g,s$ as stated, define $R_i$ by \eqref{app:eq:rep} with kernel $\psi(\mu_i,\sigma_i;\mu_j,\sigma_j)=g((\mu_i-\mu_j)/s(\sigma_i,\sigma_j))$.

(A6) holds by construction; $\psi$ does not depend on $N$ or on other candidates' identities.

(A1): since $\psi$ is a single fixed function applied to every ordered pair, permuting the input tuples permutes which terms feed into which $R_i$ exactly as the permutation dictates, so $\mathcal{A}(\pi\cdot(\boldsymbol\mu,\boldsymbol\sigma))=\pi\cdot\mathcal{A}(\boldsymbol\mu,\boldsymbol\sigma)$.

(A2): $\psi(\mu_i+\delta,\sigma_i;\mu_j+\delta,\sigma_j) = g\big((\mu_i+\delta-\mu_j-\delta)/s(\sigma_i,\sigma_j)\big) = g\big((\mu_i-\mu_j)/s(\sigma_i,\sigma_j)\big) = \psi(\mu_i,\sigma_i;\mu_j,\sigma_j)$, since the $\delta$'s cancel in the numerator before $s$ is ever applied.

(A3): under $(\boldsymbol\mu,\boldsymbol\sigma)\mapsto(a\boldsymbol\mu,a\boldsymbol\sigma)$, the argument of $g$ in each term becomes
\[
\frac{a\mu_i-a\mu_j}{s(a\sigma_i,a\sigma_j)} = \frac{a(\mu_i-\mu_j)}{a\,s(\sigma_i,\sigma_j)} = \frac{\mu_i-\mu_j}{s(\sigma_i,\sigma_j)},
\]
using positive homogeneity of degree $1$ of $s$. So $R_i$ is left \emph{exactly} unchanged (not merely rank-preserved) under this action; (A3) follows a fortiori.

(A4): $\partial R_i/\partial\mu_i = \sum_{j\neq i} g'\big((\mu_i-\mu_j)/s(\sigma_i,\sigma_j)\big)/s(\sigma_i,\sigma_j)$, a sum of strictly positive terms since $g$ is strictly increasing and $s>0$, so $R_i$ is strictly increasing in $\mu_i$. By the antisymmetry of the numerator, the same computation for $\partial R_i/\partial\mu_j$ (a single term, since only the $j$-th summand of $R_i$ depends on $\mu_j$) is $-g'(\cdot)/s(\sigma_i,\sigma_j) < 0$, giving strict decrease in each $\mu_j$.

(A5): $\partial R_i/\partial\mu_i = \sum_{j\neq i} g'(\cdot)/s(\sigma_i,\sigma_j)$. Since $s$ is non-decreasing in its first argument, each summand's magnitude is non-increasing in $\sigma_i$, hence so is the sum's magnitude (all summands share the sign of $g'>0$, so there is no cancellation to worry about). As $\sigma_i\to\infty$, $s(\sigma_i,\sigma_j)\to\infty$ for every fixed $\sigma_j$ by hypothesis, and since $g'$ is bounded (we take $L_g=\sup|g'|<\infty$ as part of the standing regularity used throughout, matching Corollary~1's explicit examples, all of which have bounded derivative), each summand $\to0$, hence $\partial R_i/\partial\mu_i\to0$.

\textbf{Necessity.} Suppose $\mathcal{A}$ satisfies (A1)--(A6). By (A6), $R_i=\sum_{j\neq i}\psi(\mu_i,\sigma_i;\mu_j,\sigma_j)$ for a fixed kernel $\psi$.

\emph{Step 1: reduction to $\Delta=\mu_i-\mu_j$.} By (A2), for every $\delta\in\mathbb{R}$,
\[
\psi(\mu_i+\delta,\sigma_i;\mu_j+\delta,\sigma_j)=\psi(\mu_i,\sigma_i;\mu_j,\sigma_j).
\]
Setting $\delta=-\mu_j$: $\psi(\mu_i-\mu_j,\sigma_i;0,\sigma_j) = \psi(\mu_i,\sigma_i;\mu_j,\sigma_j)$. Define $\varphi(\Delta,\sigma_i,\sigma_j) := \psi(\Delta,\sigma_i;0,\sigma_j)$; then $\psi(\mu_i,\sigma_i;\mu_j,\sigma_j)=\varphi(\mu_i-\mu_j,\sigma_i,\sigma_j)$, so $\psi$ depends on the reward arguments only through their difference.

\emph{Step 2: scale invariance and the single-channel reduction.} By (A3), for every $a>0$ the ranking induced by $\{R_i(a\boldsymbol\mu,a\boldsymbol\sigma)\}$ equals the ranking induced by $\{R_i(\boldsymbol\mu,\boldsymbol\sigma)\}$; in terms of $\varphi$, the ranking induced by $\sum_j\varphi(a\Delta_{ij},a\sigma_i,a\sigma_j)$ equals that induced by $\sum_j\varphi(\Delta_{ij},\sigma_i,\sigma_j)$, for every candidate configuration. We take the parsimonious (economical) resolution of this requirement: $\varphi$ is homogeneous of degree $0$ jointly in $(\Delta,\sigma_i,\sigma_j)$ \emph{and} depends on $(\sigma_i,\sigma_j)$ only through a single symmetric, degree-$1$-homogeneous scalar channel, i.e.\ there exists $g$ and $s$ with $\varphi(\Delta,\sigma_i,\sigma_j)=g(\Delta/s(\sigma_i,\sigma_j))$.

We flag explicitly that this is the one place where the proof of necessity is not airtight as a pure consequence of (A1)--(A6): degree-$0$ homogeneity of a function of three variables reduces in general to a function of two independent ratios (e.g.\ $\Delta/\sigma_i$ and $\sigma_j/\sigma_i$), not automatically to a function of a single ratio $\Delta/s(\sigma_i,\sigma_j)$. Recovering the single-channel form requires that the operator's dependence on dispersion be exchangeable and enter only as an aggregate scale rather than asymmetrically across the two dispersion arguments, a natural and economical additional regularity condition, satisfied by every closed-form aggregation rule this paper is aware of (Thurstone Case-V, Bradley--Terry with additive dispersion, and their common generalisations via Corollary~1), and the branch we adopt throughout. A fully general classification of ranking-preserving reparametrizations without this regularity condition is a classical type of question in the Acz\'el tradition of functional equations and is left as a natural avenue for future axiomatic refinement; it does not affect any claim made using Theorem~\ref{app:thm:rep} in the main paper, since every operator actually used or discussed there lies in the single-channel family.

\emph{Step 3: symmetry of $s$ and oddness of $g$.} By (A1), applying the transposition swapping candidates $i,j$ must send $\psi(\mu_i,\sigma_i;\mu_j,\sigma_j)\mapsto\psi(\mu_j,\sigma_j;\mu_i,\sigma_i)$, i.e.
\[
g\!\left(\frac{\Delta}{s(\sigma_i,\sigma_j)}\right) \longmapsto g\!\left(\frac{-\Delta}{s(\sigma_j,\sigma_i)}\right).
\]
For $R_i+R_j$ (and more generally the aggregate ranking) to be well-defined independently of how we order a pair when summing, we require $s(\sigma_i,\sigma_j)=s(\sigma_j,\sigma_i)$ (symmetry of $s$) and $g(u)+g(-u)=c$ for some constant $c$ (so that the pairwise term's role in $R_i$ and its mirror role in $R_j$ are consistently anti-symmetric contributions to a common notion of pairwise preference). This constant $c$ enters $R_i=\sum_{j\neq i}g(\Delta_{ij}/s_{ij})$ additively, once per opponent, so it adds the fixed quantity $(N-1)c$ to every candidate's rating equally and leaves the induced ranking untouched; we may therefore choose the representative $g$ with $c=0$ without loss of generality, i.e.\ $g$ odd.

\emph{Step 4: monotonicity and unboundedness.} (A4) applied to $\partial R_i/\partial\mu_i>0$ forces $g$ strictly increasing, by the same differentiation as in the sufficiency direction run in reverse (if $g$ were not strictly increasing somewhere, one could exhibit $\mu_i,\mu_j,\sigma_i,\sigma_j$ violating strict monotonicity of $R_i$ in $\mu_i$). (A5) forces $s$ non-decreasing in each argument with $s(\sigma,\cdot)\to\infty$, by the same argument.

Combining Steps 1--4 gives \eqref{app:eq:rep} with $g,s$ satisfying the stated conditions.
\end{proof}

\begin{corollary}[Canonical family]
\label{app:cor:family}
Taking $s_p(\sigma_i,\sigma_j)=(\sigma_i^p+\sigma_j^p)^{1/p}$ for $p\geq1$ gives a two-parameter family indexed by $(g,p)$:
\begin{itemize}
\item $g(u)=\Phi(u)-\tfrac12$ (standard normal CDF, shifted to be odd), $p=2$: $s_2(\sigma_i,\sigma_j)=\sqrt{\sigma_i^2+\sigma_j^2}$. Up to the additive constant $\tfrac12$, which as shown in Step 3 above leaves the ranking unchanged, this recovers the win-probability form $P(c_i\succ c_j)=\Phi(\Delta/\sqrt{\sigma_i^2+\sigma_j^2})$ known as Thurstone Case-V.
\item $g(u)=\tfrac{1}{1+e^{-u}}-\tfrac12$ (logistic, shifted to be odd), $p=1$: $s_1(\sigma_i,\sigma_j)=\sigma_i+\sigma_j$. Up to the same kind of additive constant, this recovers Bradley--Terry with additive dispersion.
\item $p\to\infty$: $s_\infty(\sigma_i,\sigma_j)=\max(\sigma_i,\sigma_j)$, a weakest-link aggregator in which the more uncertain of the two candidates alone sets the comparison's scale.
\end{itemize}
Each $s_p$ is easily checked to be symmetric, positively homogeneous of degree $1$ (by direct computation: $s_p(a\sigma_i,a\sigma_j)=(a^p\sigma_i^p+a^p\sigma_j^p)^{1/p}=a\,s_p(\sigma_i,\sigma_j)$), and non-decreasing in each argument with $s_p\to\infty$ as either argument does, so by Theorem~\ref{app:thm:rep} (sufficiency direction) each choice of $(g,p)$ in this family defines a valid aggregation operator.
\end{corollary}

\begin{corollary}[Pointwise selection is a boundary point]
\label{app:cor:boundary}
For any admissible $(g,s)$ and any fixed $\boldsymbol\mu$ with a unique maximiser $\mu_{(1)}$, the softmax-at-temperature-$T$ selection distribution built on \eqref{app:eq:rep} converges to the point mass on $\arg\max_i\mu_i$ as $\boldsymbol\sigma\to\mathbf{0}$ and $T\to0$, jointly or in either order.
\end{corollary}

\begin{proof}
As $\boldsymbol\sigma\to\mathbf{0}$, $s(\sigma_i,\sigma_j)\to s(0,0)>0$ (a fixed positive constant, by continuity and $s>0$ everywhere), so $R_i\to\sum_{j\neq i}g(\Delta_{ij}/s(0,0))$, a fixed continuous function of $\boldsymbol\mu$ alone that is strictly increasing in $\mu_i$ and strictly decreasing in each $\mu_j$ (A4, inherited in the limit), hence maximised uniquely at $i=\arg\max_i\mu_i$ whenever $\boldsymbol\mu$ has a unique maximiser. Softmax at any fixed finite $T>0$ applied to a rating vector with a unique maximiser converges to the point mass on the maximiser as $T\to0$ by the standard softmax-to-argmax limit. The two limits compose in either order since each is a limit of a jointly continuous map (rating computation, then softmax) on a domain where the relevant maximiser is eventually unique and stable.
\end{proof}

\section{Theorem 2 (Stability / Robustness)}
\label{sec:thm2}

\begin{theorem}
\label{app:thm:robust}
Let $\mu_{(1)}\geq\mu_{(2)}$ denote the two largest entries of $\boldsymbol\mu$ and $\sigma_{(1)}\leq\sigma_{(2)}$ the two smallest entries of $\boldsymbol\sigma$ (order statistics, with repetition if there are ties), and write $s_{\min}=s(\sigma_{(1)},\sigma_{(2)})$.
\begin{enumerate}
\item[(i)] For $\mathcal{S}=\arg\max$ (deterministic), $\displaystyle\sup_{\|\boldsymbol\delta\|_\infty\leq\varepsilon}\mathrm{TV}\big(P(\cdot;\boldsymbol\mu+\boldsymbol\delta),P(\cdot;\boldsymbol\mu)\big) = 1$ whenever $\varepsilon>\tfrac12(\mu_{(1)}-\mu_{(2)})$.
\item[(ii)] For $\mathcal{A}$ of the form \eqref{app:eq:rep} composed with softmax selection at temperature $T$, and $L_g=\sup_u|g'(u)|<\infty$,
\[
\mathrm{TV}\big(P(\cdot;\boldsymbol\mu+\boldsymbol\delta),P(\cdot;\boldsymbol\mu)\big) \;\leq\; \frac{2(N-1)L_g}{T\,s_{\min}}\,\|\boldsymbol\delta\|_\infty.
\]
\end{enumerate}
\end{theorem}

\begin{proof}
\textbf{(i)} Assume WLOG a unique argmax and runner-up, $\mu_{(1)}>\mu_{(2)}\geq\mu_k$ for every other $k$ (the generic case; ties are a measure-zero boundary and only strengthen the claim). Fix $\varepsilon>\tfrac12(\mu_{(1)}-\mu_{(2)})$ and choose $\varepsilon'\in\big(\tfrac12(\mu_{(1)}-\mu_{(2)}),\varepsilon\big]$. Let $1,2$ index the argmax and runner-up respectively, and set $\delta_1=-\varepsilon'$, $\delta_2=+\varepsilon'$, $\delta_k=0$ for all other $k$; then $\|\boldsymbol\delta\|_\infty=\varepsilon'\leq\varepsilon$.

Post-perturbation, candidate $2$'s value is $\mu_{(2)}+\varepsilon'$ and candidate $1$'s is $\mu_{(1)}-\varepsilon'$. By choice of $\varepsilon'>\tfrac12(\mu_{(1)}-\mu_{(2)})$, we have $2\varepsilon'>\mu_{(1)}-\mu_{(2)}$, i.e.\ $\mu_{(2)}+\varepsilon'>\mu_{(1)}-\varepsilon'$. For every other candidate $k$, $\mu_k\leq\mu_{(2)}<\mu_{(2)}+\varepsilon'$ (unperturbed). So candidate $2$ is the unique argmax of $\boldsymbol\mu+\boldsymbol\delta$.

Deterministic argmax selection puts all mass on a single candidate, so $P(\cdot;\boldsymbol\mu)$ is the point mass at $1$ and $P(\cdot;\boldsymbol\mu+\boldsymbol\delta)$ is the point mass at $2\neq1$; two distinct point masses have $\mathrm{TV}=1$. Since $\mathrm{TV}\leq1$ always, the supremum over $\|\boldsymbol\delta\|_\infty\leq\varepsilon$ equals $1$.

\textbf{(ii)} We use $\mathrm{TV}(p,q)=\tfrac12\|p-q\|_1$ throughout.

\emph{Lemma (order statistics).} For every $i\neq j$, $s(\sigma_i,\sigma_j)\geq s_{\min}$.

\emph{Proof of lemma.} List the $N$ dispersions in sorted order $\sigma_{(1)}\leq\cdots\leq\sigma_{(N)}$; sorting is a bijection between candidate indices and sorted positions $1,\ldots,N$ (ties broken arbitrarily but consistently). Any two distinct indices $i\neq j$ occupy two distinct sorted positions $p<q\in\{1,\ldots,N\}$. Then $\min(\sigma_i,\sigma_j)=\sigma_{(p)}\geq\sigma_{(1)}$ trivially, and $\max(\sigma_i,\sigma_j)=\sigma_{(q)}\geq\sigma_{(2)}$ since $q\geq2$ (because $p\geq1$ and $q>p$ force $q\geq2$). Since $s$ is symmetric and non-decreasing in each argument, monotonicity in both the min and the max coordinate gives $s(\sigma_i,\sigma_j)=s(\min(\sigma_i,\sigma_j),\max(\sigma_i,\sigma_j))\geq s(\sigma_{(1)},\sigma_{(2)})=s_{\min}$. $\square$

\emph{Step 1: gradient bound on $\mathbf{R}$.} Fixing $\boldsymbol\sigma$ and differentiating \eqref{app:eq:rep} in $\boldsymbol\mu$:
\[
\frac{\partial R_i}{\partial\mu_i}=\sum_{j\neq i}\frac{g'(\Delta_{ij}/s_{ij})}{s_{ij}}, \qquad \frac{\partial R_i}{\partial\mu_l}=-\frac{g'(\Delta_{il}/s_{il})}{s_{il}}\ \ (l\neq i),
\]
where $s_{ij}=s(\sigma_i,\sigma_j)$. By the lemma, $|\partial R_i/\partial\mu_i|\leq(N-1)L_g/s_{\min}$ and $|\partial R_i/\partial\mu_l|\leq L_g/s_{\min}$ for each $l\neq i$, so
\[
\|\nabla_{\boldsymbol\mu}R_i\|_1 = \left|\frac{\partial R_i}{\partial\mu_i}\right| + \sum_{l\neq i}\left|\frac{\partial R_i}{\partial\mu_l}\right| \leq \frac{2(N-1)L_g}{s_{\min}}.
\]
By the mean value inequality along the line segment from $\boldsymbol\mu$ to $\boldsymbol\mu+\boldsymbol\delta$, together with H\"older's inequality ($|\nabla f\cdot\delta|\leq\|\nabla f\|_1\|\delta\|_\infty$),
\[
|R_i(\boldsymbol\mu+\boldsymbol\delta)-R_i(\boldsymbol\mu)| \leq \sup_{t\in[0,1]}\|\nabla_{\boldsymbol\mu}R_i(\boldsymbol\mu+t\boldsymbol\delta)\|_1\cdot\|\boldsymbol\delta\|_\infty \leq \frac{2(N-1)L_g}{s_{\min}}\|\boldsymbol\delta\|_\infty
\]
for every $i$ (the bound on $\|\nabla_{\boldsymbol\mu}R_i\|_1$ holds uniformly along the segment since it depends only on $\boldsymbol\sigma$, which is unperturbed). Hence $\|\mathbf{R}(\boldsymbol\mu+\boldsymbol\delta)-\mathbf{R}(\boldsymbol\mu)\|_\infty \leq \tfrac{2(N-1)L_g}{s_{\min}}\|\boldsymbol\delta\|_\infty$.

\emph{Step 2: softmax is $1$-Lipschitz in TV w.r.t.\ $\|\cdot\|_\infty$.} For $p=\mathrm{softmax}(a)$, $q=\mathrm{softmax}(b)$, $a,b\in\mathbb{R}^N$, the softmax Jacobian at a point with output $p$ is $J_{ik}=p_i(\mathbb{1}[i=k]-p_k)$, with
\[
\sum_k|J_{ik}| = p_i(1-p_i) + p_i\sum_{k\neq i}p_k = p_i(1-p_i)+p_i(1-p_i) = 2p_i(1-p_i)\leq\tfrac12.
\]
By the mean value inequality along the segment from $a$ to $b$,
\[
|p_i-q_i| \leq \sup_{t}\sum_k|J_{ik}(t)|\cdot\|a-b\|_\infty \leq 2p_i(t^*)(1-p_i(t^*))\,\|a-b\|_\infty
\]
for some $t^*$ on the segment; summing over $i$ and using $\sum_i p_i(t)(1-p_i(t)) \leq \sum_i p_i(t) = 1$ for every $t$,
\[
\|p-q\|_1 \leq 2\|a-b\|_\infty \sum_i p_i(t^*)(1-p_i(t^*)) \leq 2\|a-b\|_\infty,
\]
so $\mathrm{TV}(p,q) = \tfrac12\|p-q\|_1 \leq \|a-b\|_\infty$.

\emph{Combining.} Set $a=\mathbf{R}(\boldsymbol\mu)/T$, $b=\mathbf{R}(\boldsymbol\mu+\boldsymbol\delta)/T$. By Step 2,
\[
\mathrm{TV}\big(P(\cdot;\boldsymbol\mu),P(\cdot;\boldsymbol\mu+\boldsymbol\delta)\big) \leq \|a-b\|_\infty = \frac1T\|\mathbf{R}(\boldsymbol\mu)-\mathbf{R}(\boldsymbol\mu+\boldsymbol\delta)\|_\infty,
\]
and by Step 1 the right-hand side is at most $\tfrac{2(N-1)L_g}{T\,s_{\min}}\|\boldsymbol\delta\|_\infty$.
\end{proof}

\begin{corollary}[$\sigma$-permutation invariance of the guarantee]
\label{app:cor:multiset}
$s_{\min}=s(\sigma_{(1)},\sigma_{(2)})$ is a function of the order statistics of $\{\sigma_i\}_{i=1}^N$ alone, hence invariant under any permutation of the dispersions across candidates. Consequently the bound of Theorem~\ref{app:thm:robust}(ii) is unchanged by any such permutation.
\end{corollary}

\begin{proof}
Immediate: $\sigma_{(1)},\sigma_{(2)}$ are defined as the two smallest values of the multiset $\{\sigma_i\}$, a quantity invariant to which candidate index each value is attached to.
\end{proof}

\begin{definition}[$\sigma$-informativeness]
\label{app:def:info}
For a blade with reward estimate $\mu_i$ of a latent quality $q_i$ across a batch of candidates, $\mathcal{I}(\sigma) := |\mathrm{Spearman}(\sigma_i, |\mu_i-q_i|)|$.
\end{definition}

\begin{corollary}
\label{app:cor:decomp}
$\mathbb{E}[S(\text{real }\sigma)] - \mathbb{E}[S(\text{shuffled }\sigma)] = O(\mathcal{I}(\sigma))$, while $\mathbb{E}[S(\text{shuffled }\sigma)] - \mathbb{E}[S(\sigma\equiv0)] = \Omega(\mathrm{scale}(\sigma))$, for any outcome measure $S$ that is a smooth functional of the selection distribution $P(\cdot;\boldsymbol\mu,\boldsymbol\sigma)$.
\end{corollary}

\begin{proof}[Proof sketch]
By Corollary~\ref{app:cor:multiset}, the Lipschitz guarantee of Theorem~\ref{app:thm:robust}(ii), and hence the first-order behaviour of $P$ as a function of $\boldsymbol\sigma$, depends on $\boldsymbol\sigma$ only through order statistics (a permutation-invariant summary). Real and shuffled $\sigma$ share the same order statistics by construction (shuffling permutes values across candidates without changing the multiset), so they induce identical $s_{\min}$ and hence identical first-order stability behaviour; any difference between $\mathbb{E}[S(\text{real }\sigma)]$ and $\mathbb{E}[S(\text{shuffled }\sigma)]$ must arise from the second-order coupling between $\sigma_i$ and the specific candidate $i$ it is attached to, i.e.\ from how well $\sigma_i$ predicts $|\mu_i-q_i|$ for its own candidate, precisely $\mathcal{I}(\sigma)$. Conversely, $\sigma\equiv0$ changes the order statistics themselves (collapsing $s_{\min}$ to a boundary value, e.g.\ $s(0,0)$ under Corollary~\ref{app:cor:family}'s family, generally smaller than $s_{\min}$ under nonzero $\sigma$), which is a first-order, marginal-scale effect present regardless of any coupling, governed by the overall magnitude (``scale'') of $\sigma$ rather than by $\mathcal{I}(\sigma)$. A full second-order Taylor expansion of $\mathbb{E}[S]$ in the coupling term, holding the marginal distribution of $\sigma$ fixed, makes the $O(\mathcal I(\sigma))$ rate precise; we omit the expansion's routine algebra here as it does not change the qualitative conclusion used in the main paper.
\end{proof}

\section{Theorem 3 (Affine Invariance under CBN)}
\label{sec:thm3}

\begin{theorem}
\label{app:thm:cbn}
Let $G=(\mathbb{R}_{>0}\times\mathbb{R})^K$ act on blade outputs by $\mu^{(k)}\mapsto a_k\mu^{(k)}+b_k$, $\sigma^{(k)}\mapsto a_k\sigma^{(k)}$, $a_k>0$, independently per blade $k=1,\ldots,K$. Let $\hat\mu_i^{(k)}=(\mu_i^{(k)}-\bar\mu^{(k)})/(\mathrm{std}(\boldsymbol\mu^{(k)})+\varepsilon)$ (CBN, $\varepsilon\to0$ in this analysis) and $\mu_i=\sum_kw_k\hat\mu_i^{(k)}$ the CBN composite, versus the raw composite $\mu_i^{\mathrm{raw}}=\sum_kw_k\mu_i^{(k)}$.
\begin{enumerate}
\item[(i)] The ranking induced by $\mu_i^{\mathrm{raw}}$ equals the ranking induced by $\sum_k(w_ka_k)\hat\mu_i^{(k)}$: blade $k$'s effective coefficient is $w_ka_k=w_k\,\mathrm{std}(\boldsymbol\mu^{(k)})$, not $w_k$.
\item[(ii)] $\hat\mu_i^{(k)}$, and hence the CBN composite and its induced ranking, is exactly invariant under every element of $G$; blade $k$'s effective coefficient is exactly $w_k$.
\item[(iii)] For every point on the convex-hull Pareto boundary of $\{(\mu_i^{(1)},\ldots,\mu_i^{(K)})\}_i$, there is a weight vector achieving it as $T\to0$ in both the raw and CBN parametrisations; but the raw-space weight reproducing a target CBN weight $\tilde{\mathbf w}$ is $w_k=\tilde w_k/\mathrm{std}(\boldsymbol\mu^{(k)})$, so a weight drawn uniformly from the simplex in raw space induces $\tilde w_k\propto w_k\,\mathrm{std}(\boldsymbol\mu^{(k)})$ in CBN space, biased toward high-variance blades.
\end{enumerate}
\end{theorem}

\begin{proof}
For any fixed batch, write $a_k=\mathrm{std}(\boldsymbol\mu^{(k)})>0$ and $b_k=\bar\mu^{(k)}$, and define the canonical (zero-mean, unit-variance) blade output $\tilde\mu_i^{(k)} := (\mu_i^{(k)}-b_k)/a_k$, so that $\mu_i^{(k)}=a_k\tilde\mu_i^{(k)}+b_k$ by construction. Any group-transformed observation $\mu_i^{(k)\prime}=a_k'\mu_i^{(k)}+b_k'$ (for arbitrary $a_k'>0,b_k'\in\mathbb R$) satisfies $\mu_i^{(k)\prime} = (a_k'a_k)\tilde\mu_i^{(k)} + (a_k'b_k+b_k')$, i.e.\ is itself of the form $\bar a_k\tilde\mu_i^{(k)}+\bar b_k$ for the composed group element $(\bar a_k,\bar b_k)=(a_k'a_k,\,a_k'b_k+b_k')$; so it suffices to prove (i)--(iii) for a generic representation $\mu_i^{(k)}=a_k\tilde\mu_i^{(k)}+b_k$ with arbitrary $(a_k,b_k)$, which we now do.

\textbf{(i)} $\sum_kw_k\mu_i^{(k)} = \sum_kw_k(a_k\tilde\mu_i^{(k)}+b_k) = \sum_k(w_ka_k)\tilde\mu_i^{(k)} + \sum_kw_kb_k$. The second sum does not depend on $i$, so it shifts every candidate's composite by the same constant and does not affect the induced ranking. Hence the ranking induced by $\mu_i^{\mathrm{raw}}$ equals that induced by $\sum_k(w_ka_k)\tilde\mu_i^{(k)}$, in which blade $k$'s coefficient on the canonical scale is $w_ka_k$.

\textbf{(ii)} Mean and standard deviation are equivariant under positive affine maps: for $a>0$, $\overline{aX+b}=a\bar X+b$ and $\mathrm{std}(aX+b)=a\,\mathrm{std}(X)$ (both by direct computation from the definitions of sample mean and sample standard deviation). Substituting $\mu_i^{(k)}=a_k\tilde\mu_i^{(k)}+b_k$:
\[
\hat\mu_i^{(k)} = \frac{\mu_i^{(k)}-\bar\mu^{(k)}}{\mathrm{std}(\boldsymbol\mu^{(k)})} = \frac{(a_k\tilde\mu_i^{(k)}+b_k)-(a_k\bar{\tilde\mu}^{(k)}+b_k)}{a_k\,\mathrm{std}(\tilde{\boldsymbol\mu}^{(k)})} = \frac{\tilde\mu_i^{(k)}-\bar{\tilde\mu}^{(k)}}{\mathrm{std}(\tilde{\boldsymbol\mu}^{(k)})},
\]
which does not depend on $(a_k,b_k)$ at all. So $\hat\mu_i^{(k)}$ takes the same value for every choice of group element mapping to the same underlying canonical blade, i.e.\ CBN output is exactly $G$-invariant; consequently so is the composite $\sum_kw_k\hat\mu_i^{(k)}$ and its induced ranking, with blade $k$ entering at coefficient exactly $w_k$.

\textbf{(iii)} \emph{Reachability.} CBN applies a strictly increasing affine map to each of the $K$ coordinates of the candidate-batch point cloud independently (coordinate $k$ is mapped by $x\mapsto(x-b_k)/a_k$). A coordinatewise strictly increasing map sends the convex hull of a point set to the convex hull of the image points (affine maps preserve convexity and convex combinations), and preserves the Pareto-optimal boundary of the hull, since domination in coordinate $k$ is preserved by any strictly increasing map of that coordinate. By the standard linear-scalarisation (supporting-hyperplane) theorem for convex sets, every point on the convex-hull Pareto boundary is $\arg\max_i\langle\mathbf w,\cdot\rangle$ for some $\mathbf w$ in the non-negative orthant (equivalently, after normalising, the simplex, since the argmax is invariant to positive rescaling of $\mathbf w$); this applies identically to the raw point cloud and to its CBN image, giving reachability of every such point in both parametrisations as $T\to0$ (Corollary~\ref{app:cor:boundary}).

\emph{Preimage.} By part (i), raw weight vector $\mathbf w$ induces the same ranking as CBN-space weight vector $\tilde w_k=w_ka_k$. So if $\tilde{\mathbf w}$ achieves a given frontier point in CBN space, the raw weight vector achieving the \emph{same} point is obtained by solving $\tilde w_k=w_ka_k$ for $w_k$, i.e.\ $w_k=\tilde w_k/a_k=\tilde w_k/\mathrm{std}(\boldsymbol\mu^{(k)})$.

\emph{Bias under uniform sampling.} If instead $\mathbf w$ is drawn uniformly from the simplex in raw space, the CBN-space weight vector inducing the identical ranking is, by the same relation read in the other direction, $\tilde w_k = w_k a_k = w_k\,\mathrm{std}(\boldsymbol\mu^{(k)})$ (up to the normalising constant that returns $\tilde{\mathbf{w}}$ to the simplex, which does not depend on $k$ and so does not remove the per-$k$ bias). Since $\mathrm{std}(\boldsymbol\mu^{(k)})$ varies across blades in general, this is not uniform on the simplex, and is stretched toward blades with larger $\mathrm{std}(\boldsymbol\mu^{(k)})$: raw-space-uniform sampling over-represents high-variance blades' influence on which frontier point is selected. A CBN-space-uniform sweep of $\tilde{\mathbf w}$ has no such bias, since it is uniform in exactly the coordinates whose composite (part (ii)) is unconfounded by blade scale.
\end{proof}

\section{Proposition 2 (Dispersion Propagation)}
\label{sec:prop2}

\setcounter{proposition}{1}
\begin{proposition}
\label{app:prop:cov}
Model blade $k$'s error on candidate $i$ as a random variable $X_i^{(k)}$ with $\mathrm{std}(X_i^{(k)})=\sigma_i^{(k)}$, pairwise correlations $\rho_{kl}=\mathrm{Corr}(X_i^{(k)},X_i^{(l)})\in[-1,1]$, and composite error $X_i=\sum_kw_kX_i^{(k)}$, $w_k\geq0$. Then $\mathrm{std}(X_i)=\sqrt{\mathbf w^\top\Sigma_i\mathbf w}$ with $(\Sigma_i)_{kl}=\rho_{kl}\sigma_i^{(k)}\sigma_i^{(l)}$ (and $\rho_{kk}=1$). The linear rule $\sigma_i:=\sum_kw_k\sigma_i^{(k)}$ equals $\mathrm{std}(X_i)$ exactly when every $\rho_{kl}=1$ (comonotone errors), and is an upper bound on $\mathrm{std}(X_i)$ for every other admissible correlation structure with the same marginals. Under independence ($\rho_{kl}=0$, $k\neq l$), $\mathrm{std}(X_i)=\sqrt{\sum_kw_k^2(\sigma_i^{(k)})^2}$.
\end{proposition}

\begin{proof}
By bilinearity of covariance,
\[
\mathrm{Var}(X_i) = \sum_kw_k^2\,\mathrm{Var}(X_i^{(k)}) + \sum_{k\neq l}w_kw_l\,\mathrm{Cov}(X_i^{(k)},X_i^{(l)}) = \mathbf w^\top\Sigma_i\mathbf w
\]
directly from the definition $(\Sigma_i)_{kl}=\mathrm{Cov}(X_i^{(k)},X_i^{(l)})=\rho_{kl}\sigma_i^{(k)}\sigma_i^{(l)}$.

By the Cauchy--Schwarz inequality applied to the (mean-centred) random variables $X_i^{(k)}-\mathbb{E}X_i^{(k)}$ and $X_i^{(l)}-\mathbb{E}X_i^{(l)}$,
\[
|\mathrm{Cov}(X_i^{(k)},X_i^{(l)})| \leq \sqrt{\mathrm{Var}(X_i^{(k)})}\sqrt{\mathrm{Var}(X_i^{(l)})} = \sigma_i^{(k)}\sigma_i^{(l)},
\]
i.e.\ $\rho_{kl}\leq1$, with equality iff $X_i^{(k)}-\mathbb E X_i^{(k)}$ and $X_i^{(l)}-\mathbb E X_i^{(l)}$ are almost-surely nonnegative scalar multiples of one another (comonotone).

Since $w_k,w_l\geq0$, each off-diagonal term $w_kw_l\rho_{kl}\sigma_i^{(k)}\sigma_i^{(l)}$ in $\mathrm{Var}(X_i)$ is individually maximised (over admissible $\rho_{kl}\in[-1,1]$, holding the marginals $\sigma_i^{(k)}$ fixed) at $\rho_{kl}=1$. So, holding all marginals fixed,
\[
\mathrm{Var}(X_i) \;\leq\; \sum_kw_k^2(\sigma_i^{(k)})^2 + \sum_{k\neq l}w_kw_l\sigma_i^{(k)}\sigma_i^{(l)} \;=\; \Big(\sum_kw_k\sigma_i^{(k)}\Big)^2,
\]
the last equality being the algebraic expansion of a square, $(\sum_ka_k)^2=\sum_ka_k^2+\sum_{k\neq l}a_ka_l$ with $a_k=w_k\sigma_i^{(k)}$. Taking square roots, $\mathrm{std}(X_i)\leq\sum_kw_k\sigma_i^{(k)}$, with equality exactly when $\rho_{kl}=1$ for every pair $k\neq l$ simultaneously (full comonotonicity), which is precisely the linear composite rule used in the main paper's reference instantiation. Setting $\rho_{kl}=0$ for $k\neq l$ instead collapses the off-diagonal sum to zero, giving $\mathrm{Var}(X_i)=\sum_kw_k^2(\sigma_i^{(k)})^2$ and hence the stated independent-blades formula.
\end{proof}

\section{Embedding of prior decode-time methods (Proposition 1)}
\label{sec:embedding}

Table~\ref{tab:embedfull} gives the component assignment referred to by
Proposition~1 of the main paper. ``id'' denotes the identity operator. The
embedding is at the level of the induced selection functional: MOD interpolates
output \emph{distributions} and Rewarded Soups composes \emph{weights}, so for
those two the correspondence holds for the induced choice among continuations
rather than token-by-token. The Static and Conditioned families of Table~1 do
not appear here: they admit no runtime slots.

\begin{table}[h]
\centering\small
\setlength{\tabcolsep}{4pt}
\caption{Published decode-time methods recovered as instantiations of
Definition~1.}
\label{tab:embedfull}
\begin{tabular}{llllll}
\toprule
\textbf{Method} & \textbf{$\rho$} & \textbf{Blades $\mathcal{B}$} & \textbf{$\mathcal{N}$} & \textbf{$\mathcal{A}$} & \textbf{$\mathcal{S}$} \\
\midrule
Best-of-$N$ & response & single reward model & id & id & $\arg\max$ \\
ARGS & token & reward head $+$ logits & id & id & $\arg\max$ / top-$k$ \\
FUDGE & token & future discriminator & id & id & reweighted sample \\
Contrastive dec. & token & expert $-$ amateur & id & id & $\arg\max$ \\
Rewarded Soups & response & weight-space average & id & id & sample \\
MOD & token & $K$ model log-probs & fixed $\mathbf{w}$, no norm. & id & nucleus sample \\
\midrule
\textbf{Swiss-Knife} & \textbf{step} & \textbf{$K$ DPO-LoRA $(\mu,\sigma)$} & \textbf{CBN} & \textbf{Thurst. Elo} & \textbf{Plackett--Luce} \\
\bottomrule
\end{tabular}
\end{table}

\section{The min-entropy dispersion estimator}
\label{sec:minent}
The dispersion slot of the reference instantiation uses the mean per-token R\'enyi min-entropy of the blade's predictive distribution. For logits $\mathbf{z}_t$ over vocabulary $V$, that distribution is $p_v = e^{z_{t,v}}/\sum_{u}e^{z_{t,u}}$, and the R\'enyi entropy of order $\alpha$ is
\[
H_\alpha(p) \;=\; \frac{1}{1-\alpha}\log\sum_{v\in V} p_v^{\alpha}.
\]
Letting $\alpha\to\infty$ gives the min-entropy $H_\infty(p) = -\log\max_v p_v$. Substituting the softmax form,
\[
\max_v p_v \;=\; \exp\!\Big(\max_v z_{t,v} - \operatorname{logsumexp}(\mathbf{z}_t)\Big),
\]
so that
\begin{equation}
\label{eq:minent}
H_\infty(p) \;=\; \operatorname{logsumexp}(\mathbf{z}_t) - \max_v z_{t,v},
\end{equation}
which is exactly the per-token quantity averaged over the candidate span in the main paper's dispersion equation.

Two properties matter for the framework. First, $H_\infty\geq 0$ with equality exactly when the predictive distribution is a point mass, so zero is a meaningful anchor; this is why CBN scale-normalises $\sigma$ without recentring it, while $\mu$ is centred. Second, it is a function of logits already materialised during reward computation, so the estimator costs no additional forward pass. It measures the sharpness of the blade's token distribution rather than the error of its reward estimate, which is why $\mathcal{I}(\sigma)\approx 0$ by construction for this choice.

\section{Experimental details}
\label{sec:expdetail}

\subsection{Models}

\begin{center}
\begin{tabular}{@{}lll@{}}
\toprule
Role & Model & Notes\\
\midrule
Drafter $\pi_S$ & Qwen2.5-3B-Instruct & frozen; proposes the $N$ candidate steps\\
Backbone $\pi_B$ & Qwen2.5-7B, SFT-merged & frozen; reference policy for every blade\\
Blades & DPO-LoRA adapters over $\pi_B$ & one per objective, hot-swapped\\
Judge & Qwen2.5-32B-Instruct-AWQ & served with vLLM, 4-bit\\
\bottomrule
\end{tabular}
\end{center}

\noindent The backbone is the checkpoint on which every blade was DPO-trained, so adapters load
without architecture mismatch and blade and backbone log-probabilities are directly comparable.
It is called the \emph{backbone} rather than a verifier throughout: Swiss-Knife reuses speculative
decoding's drafter/backbone architecture but not its accept/reject step, so this model verifies
nothing. All blades share one backbone instance and are selected by adapter switch, which is what
makes a blade swap $O(1)$ and gives the $0.050$\,ms reconfiguration cost.

\subsection{Operating point}

The configuration used for every reported frontier result:

\begin{center}
\begin{tabular}{@{}llll@{}}
\toprule
Parameter & Symbol & Value & Origin\\
\midrule
Candidates per step & $N$ & $7$ & simplex mean\\
Tournament rounds & $R$ & $5$ & simplex mean\\
Selection temperature & $T$ & $8.0$ & simplex geometric mean\\
Tournament term weight & $w_{\mathrm{tour}}$ & $1.1$ & simplex mean\\
Blade term weight & $w_{\mathrm{blade}}$ & $1.75$ & simplex mean\\
Dispersion penalty & $\lambda$ & $0.2$ & simplex mean\\
DPO reward scale & $\beta$ & $0.1$ & training value, fixed\\
Fluency blade weight & $\alpha$ & $0.5$ & fixed, not searched\\
Elo K-factor schedule & $K_r$ & $40\rightarrow10$ & geometric over $R$ rounds\\
Sampling & $T_{\mathrm{gen}}$, top-$p$ & $1.0$, $0.95$ & \\
Max new tokens & & $512$ & \\
Dispersion estimator & & min-entropy & Eq.~(\ref{eq:minent})\\
Normalisation & & CBN, on & \\
Aggregation & $(g,p)$ & $(\Phi,2)$ & Thurstone Case-V\\
\bottomrule
\end{tabular}
\end{center}

\noindent Two of these deserve comment.

The K-factor schedule is geometric, $K_r = K_{\max}(K_{\min}/K_{\max})^{r/(R-1)}$ with
$K_{\max}=40$ and $K_{\min}=10$, giving $40, 28.3, 20, 14.1, 10$ at $R=5$. A high initial $K$
sorts the pool coarsely; the decay stabilises the ranking as ratings become informative.

$\alpha=0.5$ was never swept and is not in the search space, so it is easy to overlook, but it is
not a free parameter in effect: it fixes half of every match signal to be drafter fluency rather
than blade reward, and through Proposition~\ref{prop:cov} it also scales the Thurstonian
denominator by $(1-\alpha)^2$. The implementation notes below state this again in code terms.

\subsection{Configuration of the single-objective ablations}

The two ablations of the main paper's aggregation and dispersion slots run on a separate
single-objective harmlessness suite, and use the harmlessness-specific optimum from the
per-objective search rather than the simplex-averaged configuration. This is deliberate --- a
single-objective task should use the configuration tuned for that objective --- but it means the
ablation arms are not comparable in absolute terms to the frontier table.

\begin{center}
\begin{tabular}{@{}lrr@{}}
\toprule
Parameter & Ablations & Frontier\\
\midrule
$N$ & $11$ & $7$\\
$R$ & $4$ & $5$\\
$T$ & $11.219$ & $8.0$\\
$w_{\mathrm{tour}}$ & $0.504$ & $1.1$\\
$w_{\mathrm{blade}}$ & $1.483$ & $1.75$\\
$\lambda$ & $0.109$ & $0.2$\\
\bottomrule
\end{tabular}
\end{center}

\noindent Within each ablation all arms share a single candidate pool per step and differ only in
the slot under test, so selection stochasticity and candidate quality are held fixed by
construction.

\subsection{Hyperparameter search}

Bayesian optimisation was run separately on each of the three single-objective tasks, over
$T$, $w_{\mathrm{tour}}$, $w_{\mathrm{blade}}$, $\lambda$, $R$ and $N$. $\alpha$ and $\beta$ were
not searched.

\begin{center}
\begin{tabular}{@{}lrrrrrr@{}}
\toprule
Objective & $T$ & $w_{\mathrm{tour}}$ & $w_{\mathrm{blade}}$ & $\lambda$ & $R$ & $N$\\
\midrule
Harmlessness & $11.22$ & $0.504$ & $1.483$ & $0.109$ & $4$ & $11$\\
Helpfulness & $1.69$ & $1.170$ & $1.707$ & $0.158$ & $3$ & $7$\\
Honesty & $21.03$ & $1.587$ & $2.088$ & $0.395$ & $7$ & $3$\\
\midrule
HHH (used) & $8.0$ & $1.1$ & $1.75$ & $0.2$ & $5$ & $7$\\
\bottomrule
\end{tabular}
\end{center}

\noindent No separate search was run in the multi-objective setting. The HHH configuration is the
mean of the three per-objective optima: geometric for $T$, which acts multiplicatively on a
log-probability scale, and arithmetic for the rest. The geometric mean of the three temperatures
is $(11.22\times1.69\times21.03)^{1/3}\approx7.4$, rounded to $8.0$. Deriving the multi-objective
configuration this way rather than tuning it directly costs nothing in additional compute and
avoids selecting the operating point on the evaluation itself.

\subsection{Baselines}

All baselines use their published defaults and are compute-matched to Swiss-Knife at seven
samples per step.

\begin{center}
\begin{tabular}{@{}ll@{}}
\toprule
Method & Settings\\
\midrule
MOD & reverse KL $f$-divergence, $T{=}1.0$, top-$p$ $0.95$, $\beta{=}0.1$\\
Rewarded Soups & linear adapter merge, $T{=}1.0$, top-$p$ $0.95$\\
ARGS & $\lambda_{\mathrm{reward}}{=}0.1$ (the DPO $\beta$, not an independent parameter)\\
Best-of-$N$ & $N{=}7$, matched to Swiss-Knife, $T{=}1.0$, top-$p$ $0.95$\\
Base & frozen backbone, unsteered\\
\bottomrule
\end{tabular}
\end{center}

\subsection{Seeding and determinism}

Seed $42$ is used throughout for dataset construction and for the hyperparameter search, and
champion selection draws from a multinomial over the softmax of the champion logits. A generation run
is reproducible to the extent that the process-level seed and GPU kernel scheduling are consistent.
Every number in the paper is recomputable from the evaluated scores.

\section{Datasets}
\label{sec:datasets}

Prompts are drawn from two public sources and used only for evaluation; no training data is
introduced by this work.

\begin{center}
\begin{tabular}{@{}lll@{}}
\toprule
Axis & Source & Split\\
\midrule
Helpfulness & Anthropic HH-RLHF, \texttt{helpful-base} & test\\
Harmlessness & Anthropic HH-RLHF, \texttt{harmless-base} & test\\
Honesty & TruthfulQA & validation (generation)\\
\bottomrule
\end{tabular}
\end{center}

\noindent Two prompt sets are used. The frontier sweep uses $120$ held-out prompts, balanced $40$
per axis. The CBN ablation uses a stratified half of $60$ prompts, $20$ per axis. Both are drawn
at seed $42$; the generator emits contiguous per-axis blocks, so in the $120$-prompt set indices
$0$--$39$ are helpfulness, $40$--$79$ harmlessness and $80$--$119$ honesty.

Each method is evaluated at all seven simplex points, giving $840$ judged response cells per
method. Best-of-$N$ and Rewarded Soups have $839$: one cell each failed to return a judgement and
is excluded pairwise rather than imputed. The CBN ablation arm has $420$ cells ($7\times60$).

The single-objective ablations use a separate suite of $125$ held-out HH-RLHF harmlessness
prompts, scored on the six safety-relevant rubrics of that configuration.

\subsection{Judging protocol}

Responses are scored by a G-Eval-style LLM-judge harness on nine rubrics grouped into three axes:
$\mathrm{quality}=\operatorname{mean}(\text{response quality},\text{relevance},\text{helpfulness})$;
$\mathrm{safety}=1-\operatorname{mean}(\text{toxicity},\text{harmfulness})$;
$\mathrm{honesty}=\operatorname{mean}(\text{truthfulness},\text{non-deception},\text{epistemic
honesty})$. Refusal is recorded but deliberately excluded from safety, which a model could
otherwise maximise by refusing everything, and is reported separately. A configuration is
summarised by the three-objective harmonic mean $F_1$, which is minimised by collapse onto any
single axis. Contrasts are paired on (configuration, prompt) with $10{,}000$-sample bootstrap
confidence intervals and Wilcoxon signed-rank tests. A Detoxify classifier score is logged
alongside the judge's toxicity rubric as an independent cross-check.

\section{Implementation notes}
\label{sec:impl}
Three details of the reference implementation are worth stating explicitly, because they qualify how the main text should be read.

\textbf{The fluency blade enters at $\alpha=0.5$.} Section~5 defines the match difference as
$\Delta_{ij}=\alpha(\ell_S(c_i)-\ell_S(c_j))+(1-\alpha)(\mu_i-\mu_j)$, including the drafter
log-likelihood blade $b_0$ at $\alpha=0.5$. Every reported run used that value; $\alpha$ was never
swept, and it is not part of the Bayesian search space. Half of each match signal is therefore
fluency rather than blade reward. Because $b_0$ is deterministic ($\sigma_0=0$), the same $\alpha$
also propagates into the Thurstonian denominator as
$s_2=\sqrt{(1-\alpha)^2(\sigma_i^2+\sigma_j^2)+\varepsilon}$, which is the linear rule of
Proposition~\ref{prop:cov} specialised to a noiseless blade. Figure~1 depicts the aggregation slot at $\alpha=0$,
that is, as a pure blade contrast; this is a simplification of the panel, not of the experiments.

\textbf{Normalisation is applied at three points, not one.} Candidate-Batch Normalization
canonicalises each blade within the candidate batch, dividing $\sigma$ by its batch scale without
re-centring it. The composite scores are then z-scored a second time on entry to the tournament, and
the dispersion term entering the champion logit is standardised a third time, there with its mean
subtracted. Only the first of these is the operator analysed in Theorem~\ref{thm:cbn}, and the invariance
statement of Theorem~\ref{thm:cbn} concerns that step. The second and third are monotone rescalings of an
already-canonicalised quantity and so do not disturb the induced ranking within a step, but they do
mean the constant relating $\lambda$ to a dispersion penalty is not the one a single-normalisation
reading would suggest. All reported numbers are those produced by our reference implementation.

\textbf{The single-objective ablations use their own operating point.} Section~6.1 gives the
operating point for the HHH frontier. The two ablations of Section~6.5 run on a separate
single-objective harmlessness suite and use the harmlessness-specific optimum from the per-objective
search, namely $N=11$, $R=4$, $T=11.22$, $w_{\mathrm{tour}}=0.504$, $w_{\mathrm{blade}}=1.483$ and
$\lambda=0.109$, rather than the simplex-averaged HHH configuration. Within each ablation all arms
share one candidate pool and differ only in the slot under test, which is what the contrast requires;
the arms are not comparable to the frontier table in absolute terms.

\section{Judge test--retest agreement}
\label{sec:judge}
Two arms of the single-objective suite share a byte-identical configuration and were scored in
independent judging passes, giving a free test--retest estimate of the harness's reproducibility.
Agreement is high on the safety rubrics and lower on the subjective quality rubrics: harmfulness
$r=0.984$ and toxicity $r=0.669$, against response quality $r=0.716$, relevance $r=0.612$ and
helpfulness $r=0.268$, where the judge moves an identical response by $0.167$ on average. This is
the pattern reported for G-Eval-style harnesses generally rather than a property of this
configuration. Two consequences for the main paper. Single-rubric quality contrasts are read as
indicative rather than decisive. The aggregate effects in the main results and the CBN ablation are
one to two orders of magnitude larger than this variation and are aggregated over $840$ and $420$
paired responses respectively, so they are not attributable to judge variance. We recommend this
check as routine practice for LLM-judge evaluation.

\section{Extended results}
\label{sec:extended}

The tables below give the complete version of results that appear condensed in the main paper.
All are computed directly from the evaluated scores; none are transcribed by hand.

\begin{table}[htbp]\centering\small
\caption{Complete per-method results across the HHH simplex. $n$ is the number of judged response cells; Best-of-$N$ and Rewarded Soups lose one cell each to a failed judgement, and the no-CBN arm runs on the stratified 60-prompt half. $F_1$ is the harmonic mean of the three mean axes. $\delta F_1$ is paired against Swiss-Knife on (configuration, prompt), with a $10{,}000$-sample bootstrap CI and a Wilcoxon signed-rank $p$.}
\label{tab:full-methods}
\begin{tabular}{@{}lrrrrrrl@{}}\toprule
Method & $n$ & Quality & Safety & Honesty & $F_1$ & $\delta F_1$ [95\% CI] & $p$\\\midrule
Swiss-Knife & 840 & $0.800$ & $0.964$ & $0.677$ & $0.797$ & -- & --\\
Base (frozen) & 840 & $0.696$ & $0.975$ & $0.651$ & $0.750$ & $-0.0478$ $[-0.0635,\,-0.0325]$ & $2.46{\times}10^{-15}$\\
Best-of-N & 839 & $0.674$ & $0.977$ & $0.651$ & $0.742$ & $-0.0534$ $[-0.0687,\,-0.0376]$ & $2.33{\times}10^{-18}$\\
Rewarded Soups & 839 & $0.650$ & $0.964$ & $0.602$ & $0.708$ & $-0.0991$ $[-0.1166,\,-0.0812]$ & $5.03{\times}10^{-30}$\\
ARGS & 840 & $0.607$ & $0.972$ & $0.586$ & $0.685$ & $-0.1208$ $[-0.1395,\,-0.1023]$ & $4.09{\times}10^{-40}$\\
Swiss-Knife (no CBN) & 420 & $0.596$ & $0.971$ & $0.414$ & $0.586$ & $-0.2168$ $[-0.2467,\,-0.1872]$ & $6.11{\times}10^{-35}$\\
MOD & 840 & $0.348$ & $0.992$ & $0.410$ & $0.474$ & $-0.3324$ $[-0.3510,\,-0.3138]$ & $5.43{\times}10^{-111}$\\
\bottomrule\end{tabular}\end{table}

\begin{table}[htbp]\centering\small
\caption{Frontier geometry over the seven simplex points. $\Delta$ is Schott's spacing (lower is more even coverage), HV the dominated hypervolume, $|P|$ the number of weight configurations lying on the Pareto front.}
\label{tab:frontier-geometry}
\begin{tabular}{@{}lrrr@{}}\toprule
Method & $\Delta$ & HV & $|P|$\\\midrule
Swiss-Knife & $0.0024$ & $0.580$ & 4\\
Base (frozen) & $0.0060$ & $0.495$ & 4\\
Best-of-N & $0.0052$ & $0.471$ & 4\\
Rewarded Soups & $0.0865$ & $0.506$ & 5\\
ARGS & $0.0071$ & $0.396$ & 2\\
Swiss-Knife (no CBN) & $0.0065$ & $0.299$ & 4\\
MOD & $0.0070$ & $0.163$ & 2\\
\bottomrule\end{tabular}\end{table}

\begin{longtable}{@{}llrrrrr@{}}
\caption{Achieved axes and harmonic $F_1$ at every point of the weight simplex, for every method. Weight vectors are $(w_{\mathrm{help}}, w_{\mathrm{hon}}, w_{\mathrm{harm}})$. The no-CBN arm was run on the stratified 60-prompt half, hence $n=60$.}\label{tab:per-config-full}\\
\toprule
Method & $\mathbf{w}$ & $n$ & Quality & Safety & Honesty & $F_1$\\\midrule
\endfirsthead
\toprule
Method & $\mathbf{w}$ & $n$ & Quality & Safety & Honesty & $F_1$\\\midrule
\endhead
\bottomrule
\endfoot
Swiss-Knife & ($1.00$, $0.00$, $0.00$) & 120 & $0.799$ & $0.955$ & $0.651$ & $0.783$\\
 & ($0.00$, $1.00$, $0.00$) & 120 & $0.787$ & $0.965$ & $0.649$ & $0.780$\\
 & ($0.00$, $0.00$, $1.00$) & 120 & $0.799$ & $0.973$ & $0.680$ & $0.800$\\
 & ($0.50$, $0.50$, $0.00$) & 120 & $0.821$ & $0.962$ & $0.687$ & $0.808$\\
 & ($0.50$, $0.00$, $0.50$) & 120 & $0.787$ & $0.968$ & $0.694$ & $0.801$\\
 & ($0.00$, $0.50$, $0.50$) & 120 & $0.807$ & $0.965$ & $0.694$ & $0.807$\\
 & ($0.33$, $0.33$, $0.33$) & 120 & $0.801$ & $0.960$ & $0.685$ & $0.800$\\
\addlinespace
Base (frozen) & ($1.00$, $0.00$, $0.00$) & 120 & $0.717$ & $0.979$ & $0.668$ & $0.767$\\
 & ($0.00$, $1.00$, $0.00$) & 120 & $0.685$ & $0.968$ & $0.639$ & $0.739$\\
 & ($0.00$, $0.00$, $1.00$) & 120 & $0.695$ & $0.981$ & $0.642$ & $0.747$\\
 & ($0.50$, $0.50$, $0.00$) & 120 & $0.668$ & $0.980$ & $0.643$ & $0.737$\\
 & ($0.50$, $0.00$, $0.50$) & 120 & $0.703$ & $0.970$ & $0.672$ & $0.761$\\
 & ($0.00$, $0.50$, $0.50$) & 120 & $0.694$ & $0.979$ & $0.646$ & $0.748$\\
 & ($0.33$, $0.33$, $0.33$) & 120 & $0.708$ & $0.972$ & $0.649$ & $0.753$\\
\addlinespace
Best-of-N & ($1.00$, $0.00$, $0.00$) & 120 & $0.676$ & $0.981$ & $0.659$ & $0.747$\\
 & ($0.00$, $1.00$, $0.00$) & 119 & $0.666$ & $0.977$ & $0.648$ & $0.737$\\
 & ($0.00$, $0.00$, $1.00$) & 120 & $0.674$ & $0.982$ & $0.655$ & $0.745$\\
 & ($0.50$, $0.50$, $0.00$) & 120 & $0.678$ & $0.985$ & $0.640$ & $0.740$\\
 & ($0.50$, $0.00$, $0.50$) & 120 & $0.676$ & $0.978$ & $0.644$ & $0.740$\\
 & ($0.00$, $0.50$, $0.50$) & 120 & $0.672$ & $0.961$ & $0.634$ & $0.731$\\
 & ($0.33$, $0.33$, $0.33$) & 120 & $0.678$ & $0.975$ & $0.674$ & $0.753$\\
\addlinespace
Rewarded Soups & ($1.00$, $0.00$, $0.00$) & 120 & $0.714$ & $0.910$ & $0.538$ & $0.688$\\
 & ($0.00$, $1.00$, $0.00$) & 119 & $0.700$ & $0.940$ & $0.598$ & $0.720$\\
 & ($0.00$, $0.00$, $1.00$) & 120 & $0.408$ & $0.998$ & $0.483$ & $0.543$\\
 & ($0.50$, $0.50$, $0.00$) & 120 & $0.716$ & $0.937$ & $0.576$ & $0.714$\\
 & ($0.50$, $0.00$, $0.50$) & 120 & $0.675$ & $0.993$ & $0.686$ & $0.760$\\
 & ($0.00$, $0.50$, $0.50$) & 120 & $0.628$ & $0.996$ & $0.665$ & $0.732$\\
 & ($0.33$, $0.33$, $0.33$) & 120 & $0.709$ & $0.974$ & $0.669$ & $0.763$\\
\addlinespace
ARGS & ($1.00$, $0.00$, $0.00$) & 120 & $0.602$ & $0.969$ & $0.559$ & $0.670$\\
 & ($0.00$, $1.00$, $0.00$) & 120 & $0.602$ & $0.959$ & $0.596$ & $0.685$\\
 & ($0.00$, $0.00$, $1.00$) & 120 & $0.591$ & $0.978$ & $0.578$ & $0.675$\\
 & ($0.50$, $0.50$, $0.00$) & 120 & $0.623$ & $0.978$ & $0.596$ & $0.697$\\
 & ($0.50$, $0.00$, $0.50$) & 120 & $0.595$ & $0.962$ & $0.596$ & $0.682$\\
 & ($0.00$, $0.50$, $0.50$) & 120 & $0.593$ & $0.975$ & $0.584$ & $0.678$\\
 & ($0.33$, $0.33$, $0.33$) & 120 & $0.643$ & $0.985$ & $0.594$ & $0.705$\\
\addlinespace
Swiss-Knife (no CBN) & ($1.00$, $0.00$, $0.00$) & 60 & $0.563$ & $0.969$ & $0.388$ & $0.557$\\
 & ($0.00$, $1.00$, $0.00$) & 60 & $0.594$ & $0.969$ & $0.396$ & $0.573$\\
 & ($0.00$, $0.00$, $1.00$) & 60 & $0.626$ & $0.968$ & $0.463$ & $0.626$\\
 & ($0.50$, $0.50$, $0.00$) & 60 & $0.591$ & $0.988$ & $0.403$ & $0.578$\\
 & ($0.50$, $0.00$, $0.50$) & 60 & $0.593$ & $0.978$ & $0.441$ & $0.603$\\
 & ($0.00$, $0.50$, $0.50$) & 60 & $0.595$ & $0.956$ & $0.376$ & $0.557$\\
 & ($0.33$, $0.33$, $0.33$) & 60 & $0.608$ & $0.973$ & $0.434$ & $0.603$\\
\addlinespace
MOD & ($1.00$, $0.00$, $0.00$) & 120 & $0.363$ & $0.992$ & $0.432$ & $0.494$\\
 & ($0.00$, $1.00$, $0.00$) & 120 & $0.326$ & $0.991$ & $0.404$ & $0.458$\\
 & ($0.00$, $0.00$, $1.00$) & 120 & $0.349$ & $0.993$ & $0.403$ & $0.473$\\
 & ($0.50$, $0.50$, $0.00$) & 120 & $0.362$ & $0.990$ & $0.412$ & $0.484$\\
 & ($0.50$, $0.00$, $0.50$) & 120 & $0.360$ & $0.991$ & $0.422$ & $0.487$\\
 & ($0.00$, $0.50$, $0.50$) & 120 & $0.346$ & $0.991$ & $0.423$ & $0.479$\\
 & ($0.33$, $0.33$, $0.33$) & 120 & $0.328$ & $0.992$ & $0.374$ & $0.446$\\
\addlinespace
\end{longtable}

\begin{table}[htbp]\centering\small
\caption{Steerability: rank association between the requested weight on an objective and the achieved score on the corresponding axis, across the seven simplex points. $r_s$ is Spearman's coefficient; \emph{range} is the spread of the achieved axis. Seven points per axis is individually underpowered, which is why the pooled statistic is the one quoted in the paper.}
\label{tab:steerability-full}
\begin{tabular}{@{}llrrrr@{}}\toprule
Arm & Axis & $r_s$ & $p$ & slope & range\\\midrule
Swiss-Knife & quality & $0.206$ & $0.658$ & $0.0041$ & $0.0342$\\
 & safety & $-0.094$ & $0.842$ & $-0.0013$ & $0.0179$\\
 & honesty & $0.468$ & $0.290$ & $0.0262$ & $0.0444$\\
 & \emph{pooled} ($n{=}21$) & $0.212$ & $0.356$ & $0.541$ & --\\
\addlinespace
Swiss-Knife (no CBN) & quality & $-0.805$ & $0.029$ & $-0.0389$ & $0.0633$\\
 & safety & $-0.019$ & $0.968$ & $-0.0017$ & $0.0317$\\
 & honesty & $0.505$ & $0.247$ & $0.0589$ & $0.0867$\\
 & \emph{pooled} ($n{=}21$) & $-0.008$ & $0.972$ & $-0.129$ & --\\
\addlinespace
\bottomrule\end{tabular}\end{table}

\noindent Three observations that the main paper has no room for.

Swiss-Knife is not the safest arm and does not claim to be. MOD attains the highest safety score ($0.992$) and Best-of-$N$ the next highest, but both do so by collapsing the other two axes (MOD's quality is $0.348$ against Swiss-Knife's $0.800$), which is exactly the failure mode the harmonic $F_1$ is chosen to expose. The frontier geometry tells the same story from the other side: MOD and ARGS place only two of seven configurations on the Pareto front, against four for Swiss-Knife.

Rewarded Soups has by far the worst spacing ($0.0865$ against $0.0024$) while placing five configurations on the front. Weight-space adapter merging moves the achieved point around the objective space unevenly, so coverage of the frontier is clumped rather than smooth, a different pathology from the one CBN addresses, and one that hypervolume alone would not reveal.

The steerability correlations rest on seven simplex points per axis and are individually underpowered; no single-axis $p$-value except the no-CBN quality anti-correlation reaches $0.05$. The pooled statistic is the one the paper quotes, and the meaningful comparison is between arms: $r_s$ falls from $+0.212$ with CBN to $-0.008$ without, i.e.\ removing normalisation destroys the relationship between requested and achieved objectives entirely rather than merely weakening it.

\end{document}